\documentclass[lettersize,journal]{IEEEtran}
\usepackage{amsmath,amsfonts}
\usepackage{algorithmic}
\usepackage{array}
\usepackage{textcomp}
\usepackage{stfloats}
\usepackage{url}
\usepackage{verbatim}
\usepackage{graphicx}
\def\BibTeX{{\rm B\kern-.05em{\sc i\kern-.025em b}\kern-.08em
    T\kern-.1667em\lower.7ex\hbox{E}\kern-.125emX}}
\usepackage{balance}
\usepackage{multirow}
\usepackage{url}

\usepackage{color}
\usepackage{textgreek}	
\usepackage{listings}
\usepackage{csvsimple}
\usepackage{longtable}
\usepackage{amssymb}
\usepackage{booktabs}
\usepackage[utf8x]{inputenc}

\newtheorem{theorem}{\textbf{Theorem}}
\newtheorem{proof}{\textbf{Proof}}

\usepackage{times}
\usepackage{epsfig}
\usepackage[table]{xcolor}
\usepackage{xcolor}
\usepackage{marvosym}

\usepackage{bbding}

\title{BEVUDA++: Geometric-aware Unsupervised Domain Adaptation for Multi-View 3D \\Object Detection}

\begin{document}

\author{Rongyu Zhang,~\IEEEmembership{Student Member,~IEEE,} Jiaming Liu,~\IEEEmembership{Student Member,~IEEE,} Xiaoqi Li, Xiaowei Chi, \\ Dan Wang,~\IEEEmembership{Senior Member,~IEEE,} Li Du~\IEEEmembership{Member,~IEEE,} \\Yuan Du \Letter,~\IEEEmembership{Senior Member,~IEEE,}  Shanghang Zhang \Letter ,~\IEEEmembership{Member,~IEEE}

\thanks{Copyright © 2024 IEEE. Personal use of this material is permitted. However, permission to use this material for any other purposes must be obtained from the IEEE by sending an email to pubs-permissions@ieee.org.}
\thanks{This paper is an extended version of~\cite{liu2024bevuda}, in Proceedings of the 2024 IEEE International Conference on
Robotics and Automation.}
\thanks{Corresponding authors \Letter:  Yuan Du, and Shanghang Zhang. Project leader: Jiaming Liu}
\thanks{Rongyu Zhang is with the National Key Laboratory for Multimedia Information Processing in School of Computer Science at Peking
University. He is also the dual Ph.D. student with both Nanjing University and The Hong Kong Polytechnic University. (e-mail: rongyuzhang@smail.nju.edu.cn)}
\thanks{Xiaowei Chi is with The Hong Kong University of Science and Technology. (e-mail: xiaowei.chi@connect.ust.hk)}
\thanks{Yuan Du and Li Du are with Nanjing University. (email: \{yuandu, ldu\}@nju.edu.cn)}
\thanks{Dan Wang is with the Department of Computing, The Hong Kong Polytechnic University, Hong Kong. (e-mail: csdwang@comp.polyu.edu.hk)}
\thanks{Jiaming Liu, Xiaoqi Li, Shanghang Zhang are with the National Key Laboratory for Multimedia Information Processing in School of Computer Science at Peking
University. (e-mail: \{liujiaming,clorisli\}@stu.pku.edu.cn; shanghang@pku.edu.cn)}
}

\markboth{IEEE Transactions on Circuits and Systems for Video Technology}%
{}

\maketitle
\begin{abstract}
Vision-centric Bird's Eye View (BEV) perception holds considerable promise for autonomous driving. Recent studies have prioritized efficiency or accuracy enhancements, yet the issue of domain shift has been overlooked, leading to substantial performance degradation upon transfer. We identify major domain gaps in real-world cross-domain scenarios and initiate the first effort to address the Domain Adaptation (DA) challenge in multi-view 3D object detection for BEV perception. Given the complexity of BEV perception approaches with their multiple components, domain shift accumulation across multi-geometric spaces (e.g., 2D, 3D Voxel, BEV) poses a significant challenge for BEV domain adaptation. In this paper, we introduce an innovative geometric-aware teacher-student framework, BEVUDA++, to diminish this issue, comprising a Reliable Depth Teacher (RDT) and a Geometric Consistent Student (GCS) model. Specifically, RDT effectively blends target LiDAR with dependable depth predictions to generate depth-aware information based on uncertainty estimation, enhancing the extraction of Voxel and BEV features that are essential for understanding the target domain. To collaboratively reduce the domain shift, GCS maps features from multiple spaces into a unified geometric embedding space, thereby narrowing the gap in data distribution between the two domains. Additionally, we introduce a novel Uncertainty-guided Exponential Moving Average (UEMA) to further reduce error accumulation due to domain shifts informed by previously obtained uncertainty guidance. To demonstrate the superiority of our proposed method, we execute comprehensive experiments in four cross-domain scenarios, securing state-of-the-art performance in BEV 3D object detection tasks, e.g., 12.9\% NDS and 9.5\% mAP enhancement on Day-Night adaptation.
\end{abstract}

\section{Introduction}
The camera-based 3D object detection has attracted increasing attention in recent years, especially in the field of autonomous driving~\cite{arnold2019survey, chen2017multi, chen2016monocular,schrum2024maveric}. Nowadays, it has obtained obvious advancements driven by Bird-Eye-View (BEV) perception methods~\cite{wang2024bevrefiner, huang2021bevdet, li2022bevdepth,wang2024towards,qiao2024local,fan2023hcpvf} and large scale labeled autonomous driving datasets~\cite{geiger2012we, caesar2020nuScenes, sun2020scalability}. However, due to the vast variety of perception scenes in the real-world application~\cite{zhang2024decomposing, wang2021exploring, li2022unsupervised}, the camera-based methods can suffer significant performance degradation caused by different data distribution.

Recently, though mono-view 3D detection methods~\cite{rao2023monocular, li2022towards, li2022unsupervised} carry out studies of different camera parameters or annotation methods variation, domain adaptation (DA) problems on many real-world scenarios is still unexplored in both Mono-view~\cite{cai2020monocular, wang2021fcos3d, brazil2019m3d, ding2020learning, li2022diversity} and Multi-view\cite{wang2022detr3d, liu2022petr, liu2022petrv2, jiang2022polarformer, li2022bevformer, fang2024cross, huang2021bevdet, li2022unifying, li2022bevdepth}. Meanwhile, we discover the tremendous domain gap of the LSS-based BEV method in typical real-world cross-domain scenarios as shown in Fig. \ref{fig:intro}, which leads to inferior performance of baseline~\cite{li2022bevdepth}. Additionally, labeling target data is often impractical in real-world DA applications. Consequently, we are exploring unsupervised domain adaptation (UDA) for BEV perception.

Since Multi-view LSS-based methods~\cite{philion2020lift, huang2021bevdet, li2022bevdepth} are usually complicated and contain several components, the domain shift error accumulation on multi-geometric spaces makes BEV-oriented UDA extremely challenging: 
(1) \textit{2D images geometric space.} Since multi-view images contain abundant semantic information, it will result in a manifold domain shift when data distribution is transferred.
(2) \textit{3D Voxel geometric space.}
Voxel features that are constructed by source domain-specific image features and unreliable depth prediction will assemble more domain shifts on the target domain. 
(3) \textit{BEV geometric space.}
Due to the shift in the above spaces, the further constructed BEV feature results in an accumulation of domain shift and leads to noises for final prediction. 

To this end, we propose a BEV-oriented geometric-aware teacher-student framework, BEVUDA++, to disentangle accumulated domain shift problems, which consists of a Reliable Depth Teacher (RDT) and a Geometric Consistent Student (GCS) model. 
Since the BEV feature heavily relies on the reliability of depth~\cite{li2022bevdepth, li2022bevstereo}, RDT ingeniously fuses target LiDAR data with trustworthy depth estimations to integrate depth-aware insights. Notably, dependable depth predictions, filtered through an uncertainty estimation scheme~\cite{gal2016dropout,zhang2024efficient,zhang2024vecaf}, augment sparse LiDAR inputs and demonstrate robustness against domain shift.
Furthermore, given that features from correlated multi-space exhibit geometric consistency, we introduce GCS to collaboratively confront the compounded domain shift. In particular, GCS projects features from multiple latent spaces into a common geometric embedding space, thereby reducing the data distribution disparity between the source and target domains.
Finally, to fortify the linkage between RDT and GCS while enhancing BEVUDA++'s resilience against domain shift, we propose the Uncertainty-guided Exponential Moving Average (UEMA) strategy. This approach introduces dynamic updates as a superior alternative to conventional static model adjustments, which not only expedites the model's adjustment to new data distributions but also significantly enhances the precision of pseudo-labels.

\begin{figure}[t]
\includegraphics[width=0.48\textwidth]{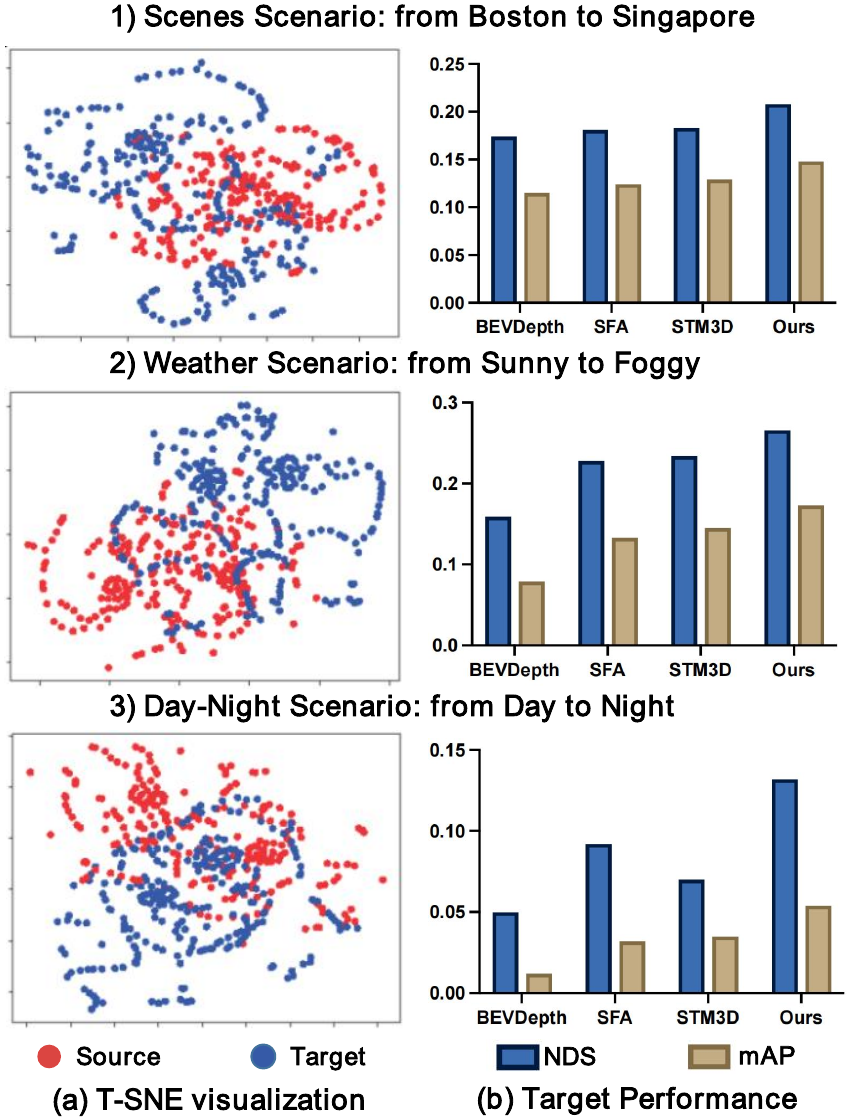}
\centering
\caption{(a) T-SNE~\cite{van2013barnes} visualization of the BEV features in different UDA scenarios, in which features are obviously separated by domain. (b) The performance of our method compared with previous works~\cite{li2022bevdepth,wang2021exploring,li2022unsupervised}. All methods are built on BevDepth with a ResNet-50 backbone.}
\label{fig:intro}
\end{figure}

We further design three classical and one continual changing UDA scenarios, which are \textbf{Scene} (from Boston to Singapore), \textbf{Weather} (from sunny to rainy and foggy), \textbf{Day-night}, and \textbf{Foggy degree} changing in~\cite{caesar2020nuScenes}. For continually changing scenarios, we construct cross-domain experiments with the continuously increased density of Fog, which gradually enlarges the domain gap. Our proposed method achieves competitive performance in all scenarios (shown in Fig. \ref{fig:intro} (b)). Compared with the previous state-of-the-art UDA method (i.e., STM3D\cite{li2022unsupervised}), it improves the NDS by 2.6\%, 3.1\%, 6.1\%, and 3.5\%, respectively, in the aforementioned four classical scenarios.
The main contributions of our proposed BEVUDA++ are summarized as follows:

\begin{itemize}
    \item We explore the unsupervised domain adaptation for BEV perception of multi-view 3D object detection. We propose a novel geometric-aware teacher-student framework, BEVUDA++, to ease the domain shift error accumulation.

    \item Within BEVUDA++, we introduce a Reliable Depth Teacher (RDT) model designed to thoroughly harness target domain insights through depth-aware information utilization and a Geometric Consistent Student (GCS) model that adeptly projects multi-latent space features into a unified geometric embedding space, effectively narrowing the data distribution gap across domains.

    \item We further devise an innovative Uncertainty-guided Exponential Moving Average (UEMA) mechanism to accelerate the model adaptation to new data distribution and sharpen pseudo-label accuracy.

\end{itemize}

\begin{figure*}[t]
\includegraphics[width=0.98\textwidth]{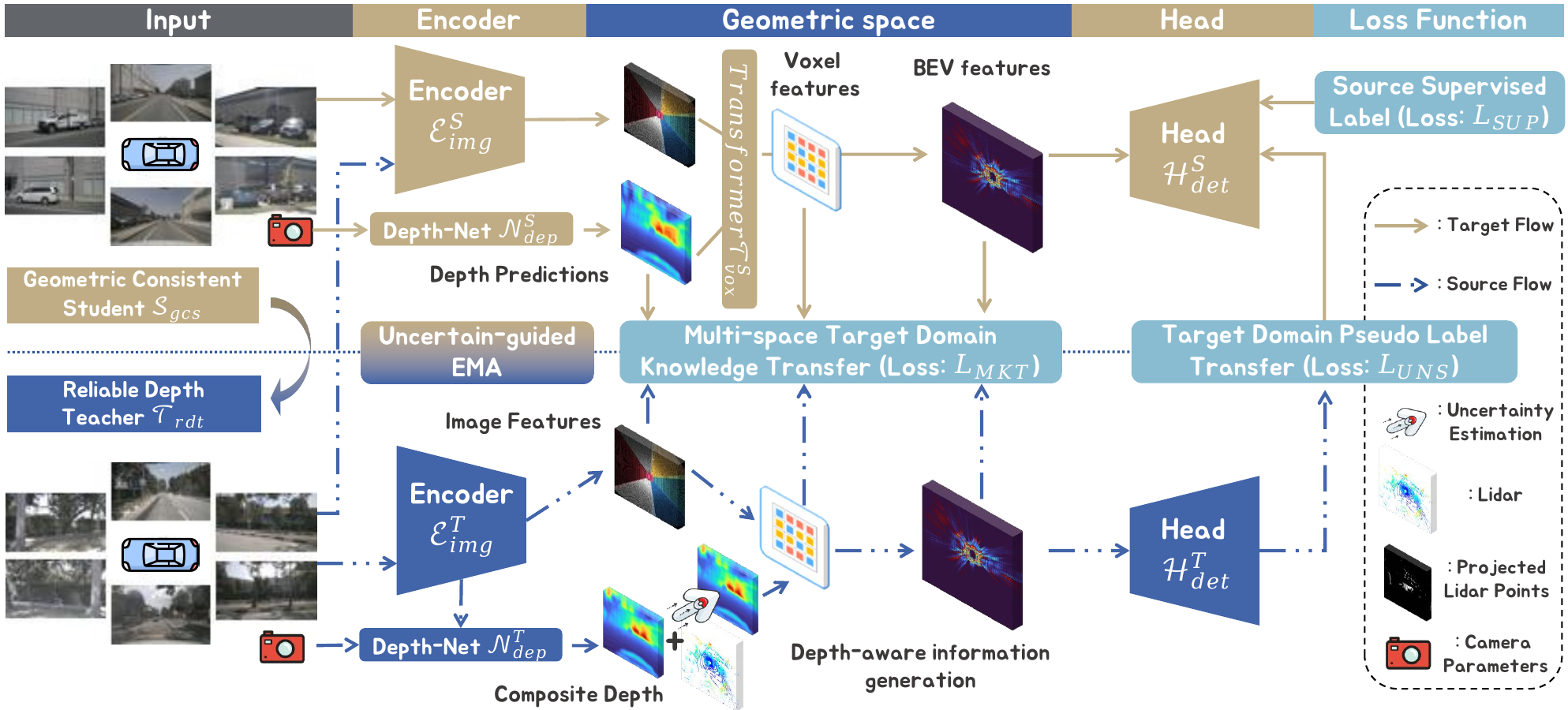}
\centering
\caption{The framework of geometric-aware teacher-student (BEVUDA++) is composed of the Reliable Depth Teacher (RDT) and Geometric Consistent Student (GCS) model. In \textbf{the bottom part}, the RDT model takes target domain input and adopts depth-aware information to construct Voxel and BEV features with sufficient target domain knowledge, which is further transferred to the student model along with detection pseudo labels. In \textbf{the upper part}, the GCS model takes two domains input and decreases 
data distribution distance in a shared geometric embedding space. \textbf{The update} of RDT is achieved by Uncertainty-guided EMA (UEMA). BEVUDA++ framework aims to comprehensively address the multi-geometric space domain shift problem.}
\label{fig:method}
\end{figure*}

\section{Related work}
\subsection{Camera-based 3D object detection.}
3D Object Detection has become a cornerstone in autonomous driving and machine scene comprehension, with two leading approaches: Single-view~\cite{cai2020monocular, wang2021fcos3d, brazil2019m3d, ding2020learning, liu2021autoshape, manhardt2019roi, barabanau2019monocular, li2022diversity, zhang2021objects, simonelli2019disentangling} and Multi-view~\cite{philion2020lift, wang2022detr3d, liu2022petr, liu2022petrv2, chen2022polar, jiang2022polarformer, li2022bevformer, reading2021categorical, huang2021bevdet, li2022unifying, li2022bevdepth, huang2022bevdet4d, li2022bevstereo}. Single-view detection branches into various methodologies, such as exploiting CAD models~\cite{liu2021autoshape, manhardt2019roi, barabanau2019monocular}, focusing on key points as prediction targets~\cite{li2022diversity, zhang2021objects}, and separating transformations for integrated 2D and 3D detection~\cite{simonelli2019disentangling}. Notably, FCOS3D~\cite{wang2021fcos3d} simultaneously predicts both 2D and 3D attributes, while D4LCN~\cite{ding2020learning} seeks to craft a more robust 3D structure by substituting 2D depth predictions with a pseudo-LiDAR approach. Moreover,~\cite{cai2020monocular} derives object depth by factoring in their actual heights. To optimally leverage depth data,~\cite{huang2022monodtr} introduces an end-to-end depth-aware transformer network. Nonetheless, with precision and practicality in focus, the trend is increasingly shifting towards proposing multi-view 3D object detectors.

The Multi-view approach to 3D object detection is primarily divided into two categories: transformer-based~\cite{carion2020end} and LSS-based~\cite{philion2020lift} systems.
Transformer-based methods, such as DETR3D~\cite{wang2022detr3d}, innovate by adapting the DETR~\cite{carion2020end} framework for 3D detection, employing a transformer network to predict 3D bounding boxes. Building upon this, subsequent studies~\cite{liu2022petr, liu2022petrv2, chen2022polar, jiang2022polarformer} utilize object or BEV grid queries~\cite{li2022bevformer} for feature extraction from images and apply attention mechanisms to facilitate the 2D-to-3D conversion. However, transformer-based models do not directly project image features into a BEV format.
On the other hand, LSS-based models~\cite{philion2020lift}, such as~\cite{reading2021categorical, huang2021bevdet, li2022unifying}, predict LiDAR depth distributions and construct a point cloud from multi-view image features for enhanced 3D detection. Notably, BEVDepth~\cite{li2022bevdepth} leverages depth supervision to refine the voxel pooling process, while BEVDet4d~\cite{huang2022bevdet4d} and BEVStereo~\cite{li2022bevstereo} integrate temporal data for comprehensive volume analysis over time. In this work, we select BEVDepth~\cite{li2022bevdepth} as our foundational 3D object detector due to its straightforward yet effective methodology and its exceptional capacity for cross-domain feature extraction.

\subsection{UDA in 3D object detection.}
Domain Adaptive Faster R-CNN~\cite{chen2018domain} pioneered the exploration of the cross-domain challenge in object detection. Drawing on the principles of unsupervised domain adaptation~\cite{ganin2015unsupervised,tian2022unsupervised,tian2023dcl}, numerous efforts~\cite{cai2019exploring,saito2019strong, wang2021exploring, xu2020exploring, xu2020cross, yu2022cross} have adopted the strategy of cross-domain alignment to examine the impact of domain shift across multi-level features. When it comes to 3D object detection, research such as~\cite{luo2021unsupervised, li2022unsupervised, zhang2021srdan} explores UDA for point cloud-based detectors, with studies like~\cite{luo2021unsupervised, zhang2021srdan} employing alignment techniques to synchronize feature and instance-level information between domains. STM3D~\cite{li2022unsupervised} further advances UDA with self-training strategies, utilizing consistent and high-fidelity pseudo labels.

Recent methods~\cite{barrera2021cycle, acuna2021towards, ng2020bev, saleh2019domain} have delved into cross-domain strategies, specifically within BEV perception, to mitigate simulation-to-reality discrepancies. In the realm of camera-based monocular 3D object detection, innovative approaches~\cite{li2022towards, li2022unsupervised} have been made to disentangle camera parameters and ensure geometric consistency during the cross-domain transition. 
In contrast, our work is committed to addressing the challenge of domain discrepancy in multi-view 3D object detection tasks, which reconstruct 3D scenes from a BEV standpoint.

\section{Methods}
\subsection{Preliminary}
\label{sec:setup}
For the UDA setting~\cite{zhao2020review}, we are provided by labeled source domain $D_{s} = \{\{I^i_{s}\}^M_{j=1},K^i_{s},G^i_{s}\}^{N_{s}}_{i=1}$ and unlabeled target domain $D_{t} = \{\{I^i_{t}\}^M_{j=1},K^i_{t}\}^{N_{t}}_{i=1}$ of $N$ samples and $M$ camera views, in which $I^i$, $K^i$, and $G^i$ denote images, camera configurations, and detection ground truth respectively. In this paper, we consider a depth-aware BEV model that operates through a sequential three-stage process for enhanced 3D perception. Specifically, the BEV model is an assembly of several components: a 2D image-view feature encoder $\mathcal{E}_{img}$\cite{he2016deep}, a depth estimation network $\mathcal{N}_{dep}$, a 3D voxel feature transformer $\mathcal{T}_{vox}$, a BEV feature decoder $\mathcal{D}_{bev}$, and a 3D detection head $\mathcal{H}_{det}$. Given a collection of multi-camera images denoted as $I^i_{t}$, the BEV model commences by encoding these images into image-view features via $\mathcal{E}_{img}$ in the following manner: 
\begin{equation}
F^{i}_{img} = \mathcal{E}_{img}(I^{i}_{t}),
\label{eq:ie}
\end{equation}
where $F^{i}_{img} \in \mathbb{R}^{C_{I}\times H \times W}$. We than estimate the images depth $F^{i}_{dep} \in \mathbb{R}^{C_{D}\times H \times W}$ with the corresponding camera configuration $K^{i}\in \mathbb{R}^{3\times 3}$ based on depth-Net $\mathcal{N}_{dep}$ as:
\begin{equation}
F^{i}_{dep} = \mathcal{N}_{dep}(F^{i}_{img},K^{i}),
\label{eq:id}
\end{equation}
where $C_{D}$ stands for the number of depth bins. The image feature $F^{i}_{img}$ and the depth feature $F^{i}_{dep}$ are then projected into BEV features $F^{i}_{vox}\in \mathbb{R}^{C_{I}\times C_{D}\times H \times W}$ with the transformer $\mathcal{T}_{vox}$ and use the 3D pooling layer $\mathcal{O}_{3dp}$ to generate the $F^{i}_{bev}\in \mathbb{R}^{C_{I}\times C'_{D}\times H' \times W'}$ with kernel size and stride of $(D_{p},H_{p},W_{p})$ and $(D_{s}, H_{s}, W_{s})$:
\begin{equation}
\begin{aligned}
F^{i}_{vox} &= \mathcal{T}_{vox}(F^{i}_{img},F^{i}_{dep}), \\
F^{i}_{bev} &= \mathcal{O}_{3dp}(F^{i}_{vox}),
\end{aligned}
\label{eq:bev}
\end{equation}
where $C'_{D} = \left \lfloor \frac{D - D_{p}}{D_{s}} \right \rfloor + 1$, with $H'$ and $W'$ calculated analogously. Furthermore, the BEV decoder $\mathcal{D}_{bev}$ then decodes the BEV feature $F^{i}_{bev}$ in conjunction with a collection of learnable BEV queries $q^{i}_{bev}$:
\begin{equation}
Q^{i}_{bev} = \mathcal{D}_{bev}(F^{i}_{bev}, q^{i}_{bev})
\label{eq:pre}
\end{equation}
where the BEV queries $q^{i}_{bev}$ engage with the BEV feature $F^{i}_{bev}$ to encapsulate the three-dimensional positional information of objects within the BEV space, resulting in the decoded query feature $Q^{i}_{bev}$. Ultimately, a 3D detection head, denoted by $\mathcal{H}_{det}$, is employed to derive the final predictions. These predictions are then evaluated against the source domain labels $G_{s}^{i}$ using the detection loss function $\ell_{det}$:
\begin{equation}
\label{eq:sup}
\mathcal{L}_{SUP} = \ell_{det}(\mathcal{H}_{det}(Q^{i}_{bev}), G_{s}^{i})
\end{equation}
where $\mathcal{L}_{SUP}$ represents the supervision loss, which quantifies the discrepancy between the predicted outputs and the ground truth labels from the source domain.

\subsection{Motivation and Overall Framework}
\label{sec:overall}
Recent BEV perception methods~\cite{carion2020end, philion2020lift, li2022bevdepth} mainly focus on improving accuracy or efficiency but ignore commonly existing domain shift problems in the real-world, resulting in performance degradation on different data distribution. This thus motivates us to carefully explore the DA problem for BEV perception. Compared with existing DA tasks, LSS-based BEV perception approaches~\cite{philion2020lift, huang2021bevdet, li2022bevdepth} are relatively complicated and contain more geometric spaces (2D image, 3D voxel, and BEV space), thus increasing domain shift error accumulation on the target domain. 
As shown in Fig. \ref{fig:intro} (a), we utilize T-SNE~\cite{van2013barnes} to visualize the data distribution distance between the source and target domain in the BEV feature. As we can see, the tremendous distribution distance reflects that the BEV feature (last latent space) assembles the accumulated domain shift from all previous spaces. 

\begin{figure}[t]
\includegraphics[width=0.48\textwidth]{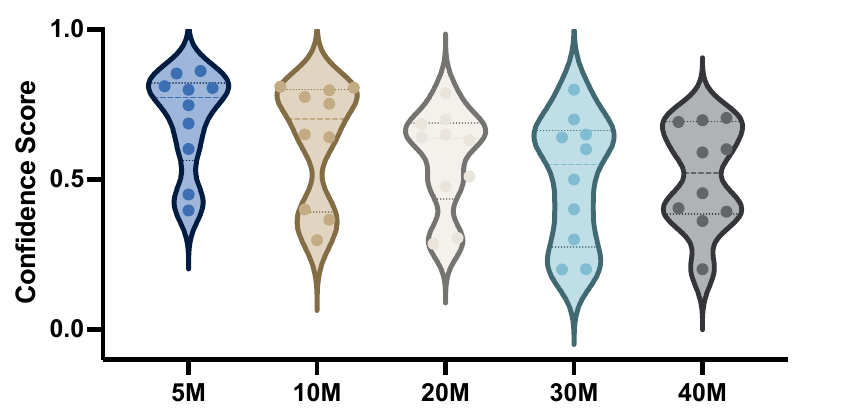}
\centering
\caption{The ten times (with Dropout) confidence scores for six pixels which are located at different M distances. M indicates meter.}
\label{fig:conf}
\end{figure}


Therefore, inspired by the observation that mean teacher predictions often exhibit higher quality than standard models~\cite{tarvainen2017mean}, we leverage a teacher-student framework to maintain stability in target domains. As shown in Fig .\ref{fig:method}, we propose a BEV-oriented geometric-aware teacher-student framework, BEVUDA++, which aims to ease the domain shift error accumulation in the target domain. It consists of an RDT model that tactfully combines target LiDAR and reliable depth prediction to compose depth-aware information under uncertainty guidance~\cite{gal2016dropout,zhang2023unimodal} towards more robust EMA teacher model update, and a GCS model to jointly address the accumulated domain shift error in multi-latent space. 



\subsection{Reliable Depth Teacher}
\label{sec:DAT}

The reliability of depth plays an important role in LSS-based BEV perception since it constructs the crucial BEV feature with 2D features~\cite{li2022bevdepth,li2022bevstereo}. Though confidence is a straightforward measurement to reflect reliability, it is trustless in pixel-wise cross-domain scenarios. As shown in Fig. \ref{fig:conf}, for each image pixel, we calculate the confidence score over time and find it fluctuates irregularly. Since domains shift impact seriously on the source model, the confidence score can no longer reflect a stable probability of the pixel being predicted correctly. We thus intend to obtain more reliable depth information in the target domain with the Reliable Depth Teacher (RDT) model, which aims to leverage depth-aware information to construct Voxel and BEV features with sufficient target domain knowledge. We enhance the teacher model's depth information by integrating sparse LiDAR data $\{L_{ldr}\}_{i=1}^{N}$ with reliable depth predictions.


Though the target LiDAR data can reflect the accurate depth information for each object, it is sparse and can lead to insufficient semantic knowledge in constructed BEV features. Therefore, we explore a new solution in dense prediction domain adaptation tasks to reduce the noise in training. Inspired by the efficacy of Bayesian deep learning~\cite{kendall2017uncertainties}, we employ an uncertainty estimation mechanism that is grounded in Bayesian inference to identify dependable depth predictions. For a hypothetical Bayesian network $\theta$, where all parameters are treated as random variables, the direct computation of the posterior distribution is typically infeasible. Consequently, we can utilize variational inference~\cite{blei2017variational} for approximation purposes. Given an input $X$, the output $Y$’s predictive distribution is obtained through $m$ stochastic forward passes, with network parameters drawn from an approximating variational distribution $q(\theta)$:
\begin{equation}
\begin{aligned}
p(Y|X) &= \int p(Y|X, \theta)q(\theta) d\theta \\
&\approx \frac{1}{m}\sum_{n=1}^{m} p(Y|X, \theta^{n}), \ \ \  \theta^{t} \sim q(\theta),
\end{aligned}
\end{equation}
where $p(Y|X, \theta^{n})$ is one forward pass with model parameter $\theta^{n}$. However, since it is incapable to directly use variational inference to estimate the posterior distribution under our neural network-based BEVUDA++ framework, therefore, we adopt the MC Dropout method~\cite{gal2016dropout} to perform Bayesian inference without changing model structure and parameter in BEVUDA++ by sampling model parameters $m$ times (e.g., $m=5$) forward propagation and obtain $m$ ensembles for each pixel. We calculate the uncertainty map of the depth prediction and figure out how it is influenced by domain shift:
\begin{equation}
\mathcal{U} (X_p) =  \left( \frac{1}{m} \sum_{n=1}^m \|p_t(Y_p|X_p) - \mu \|^2 \right) ^{\frac{1}{2}},
\label{eq:mc}
\end{equation}
where $p_n(Y_p|X_p)$ denotes the probability of the input pixel $X_p\in I^{i}_{s}$ during the current $n^{th}$ forward pass, while $\mu$ represents the average probability across $m$ rounds for $X_p$ and $Y_p\in G^{i}_{s}$. The term $\mathcal{U} (X_p)$ encapsulates the uncertainty of the source model with respect to the pixel-wise target input $X_p$. Those predicted depth points that exhibit an uncertainty below a predefined threshold $\Theta$ are retained, as they demonstrate greater adaptability within the target domain. Conversely, depth points that do not meet this criterion are supplanted by the corresponding LiDAR $\{L_{ldr}\}_{i=1}^{N}$:
\begin{equation}
\left\{
\begin{aligned}
    & F^{p}_{dep} = F^{p}_{dep}, \quad \mathcal{U} (X_p) \leq \Theta \\
    & F^{p}_{dep} = L^{p}_{ldr}, \quad \mathcal{U} (X_p) > \Theta \quad 
\end{aligned}
\right.
\end{equation}

Consequently, RDT leverages depth-aware information to more effectively distill knowledge pertinent to the target domain, which is subsequently imparted to the student model to facilitate cross-domain learning.

\subsection{Geometric Consistent Student}
\label{sec:MFA}
We introduce the Geometric Consistent Student (GCS) to maximize the utility of the transferred knowledge while simultaneously mitigating the cumulative domain shift. Leveraging geometric consistency, we project features from multi-geometric spaces, such as 2D images, voxels, and BEV space, into a unified geometric embedding space. This process reduces the disparity in data distribution between the source and target domains. Concretely, within the source domain, we employ individual Multilayer Perceptrons (MLPs) to transform the three types of features into a communal embedding space. The resulting embedding has dimensions $3 \times C \times n$, where the channel dimension $C=256$, and $n$ represents the number of object categories. Next, we reshape the feature dimensions to $768 \times n$ and utilize a collective MLP to consolidate the categorical features into a prototype specific to the source domain, sized at $256 \times n$. Following the same procedure, we generate a corresponding prototype for the target domain.
To align the prototypes from both domains, we employ an alignment loss $\mathcal{L}_{ALI}$ as delineated by~\cite{ganin2016domain}. $\mathcal{L}_{ALI}$ is defined in Eq. \ref{eq:MFA}, which aligns the two domain prototypes and effectively bridges the gap between them:
\begin{equation}
\label{eq:MFA}
\mathcal{L}_{ALI}(F_{s},F_{t}) = \log{\mathcal{D}_{dis}}(F_{s}) + \log(1-\mathcal{D}_{dis}(F_{t}))
\end{equation}
where $F_{s}$ and $F_{t}$ denote the prototype representations of the source and target domains, respectively, while $\mathcal{D}_{dis}$ serves as the domain discriminator.


Additionally, we extend the comprehensive transfer of multi-space knowledge from the teacher model $\mathcal{T}_{rdt}$ to the student model $\mathcal{S}_{gcs}$, reinforcing consistency within the unified geometric space. This process is governed by the knowledge transfer loss $\mathcal{L}_{MKT}$, which is defined as follows:
\begin{equation}
\mathcal{L}_{MKT} = \sum_{l\in L}\frac{1}{W_{l}'\times H_{l}'}\sum_{p\in P}||F_{s,l}^{{p}}-F_{t,l}^{{p}}||^{2}
\label{eq:KT}
\end{equation}
where $F_{s,l}^{p}$ and $F_{t,l}^{p}$ represent the value of the $p^{th}$ pixel from the source and target domains, respectively, within the $l^{th}$ geometric space, where $l$ encompasses the set $\{\text{2D images, 3D Voxel, BEV}\}$. The terms $W_{l}^{'}$ and $H_{l}^{'}$ denote the width and height of the feature maps being transferred. The set $P$ encompasses all pixel positions, given by $\{1, 2, ..., W_{l}' \times H_{l}'\}$. In the GCS model, the alignment loss $\mathcal{L}_{ALI}$ and knowledge transfer loss $\mathcal{L}_{MKT}$ collaboratively improved feature alignment and significantly diminished the domain divergence.

\subsection{Uncertainty-guided Exponential Moving Average}
\label{sec:ema}
The rest of the teacher model is built with exponential moving average (EMA)~\cite{tarvainen2017mean}. The initial weights of the teacher model and student model are loaded from the source domain pre-trained model. The standard EMA approach employs a smoothing coefficient $\alpha$ to progressively update the teacher model's parameters based on those of the student model, using exponentially decaying weights to enhance stability and performance. However, a static $\alpha$ does not account for the dynamic nature of model updates in evolving domain shifts. To address this, we harness the calculated estimated uncertainty mean value $\mathcal{U}^{t}$ from the student model $\mathcal{S}_{gcs}$ at training iteration $t$, integrating it into the EMA to facilitate an adaptive update of the teacher model $\mathcal{T}_{rdt}$ for generating more reliable pseudo labels to guide the student model to study the data distribution of the target domain more effectively. Thus, the modified EMA equation can be presented as:
\begin{equation}
\label{eq:uema}
\begin{aligned}
     \mathcal{T}_{rdt}^{t} = &(\alpha+\sigma\times\mathcal{U}^{t-1})\cdot\mathcal{T}_{rdt}^{t-1} \\
     &+ (1- \alpha - \sigma\times\mathcal{U}^{t-1}) \cdot\mathcal{S}_{gcs}^{t}
\end{aligned}
\end{equation}
where $\alpha$=0.999 and $\sigma$=0.001. Simultaneously, we employ classification confidence as a criterion for pseudo-label generation, mitigating the risk of error propagation within the RDT model.

The uncertainty value $\mathcal{U}^t$ tends to fluctuate during the initial stages of training, subsequently stabilizing as the model develops greater resilience to domain shifts. When $\mathcal{U}^t$ is elevated at a particular iteration, it indicates a heightened level of uncertainty within the RDT model, casting doubt on the reliability of the pseudo-labels it produces. Consequently, the GCS model, which is updated based on these labels, may also be compromised in its reliability. In such cases, it is prudent to restrict the weight updates to the RDT model to prevent the propagation of unreliable information. Conversely, when uncertainty is low, indicating a more confident model stance, the weight updates can be more liberally applied. This adaptive approach ensures that the teacher-student framework responds appropriately to the model's evolving confidence, thereby maintaining the integrity of the learning process. 

\subsection{Training objectives and inference}
\label{sec:loss}
In the overall framework, RDT and GCS jointly address the domain shift with UEMA in multi-geometric space and avoid domain shift accumulation.  
Following~\cite{li2022bevdepth}, we use the same detection loss ($\mathcal{L}_{UNC}$) to update the student model, which is penalized by the target domain pseudo label.
Meanwhile, the integrated domain adaptation loss $\mathcal{L}_{DA}$ is shown in Eq.\ref{eq:domain}.
\begin{equation}
    \mathcal{L}_{DA} =   \lambda_1*\mathcal{L}_{UNC} + \lambda_2*\mathcal{L}_{SUP} + \lambda_3*\mathcal{L}_{MKT} + \lambda_4*\mathcal{L}_{ALI},
\label{eq:domain}
\end{equation}
where $\mathcal{L}_{SUP}$ represents the detection loss as defined by~\cite{li2022bevdepth}, which is calculated using the source domain detection labels as outlined in Eq.\ref{eq:sup}. To ensure equilibrium among the loss penalties, we assign the weights $\lambda_1$ and $\lambda_2$ a value of 1, while $\lambda_3$ and $\lambda_4$ are adjusted to 0.1. During the inference phase, in line with other camera-based approaches~\cite{li2022bevdepth,li2022bevformer,li2022bevstereo}, our method exclusively utilizes data from multi-view cameras.

\section{Justification of inter-domain divergence}
\label{sec:just}
To validate the effectiveness of BEVUDA++, we assess its domain representation by calculating the distribution distance using Ben-David's domain distance definition~\cite{ben2006analysis, ben2010theory} and the $\mathcal{H}$-divergence metric, extending prior domain transfer research~\cite{ganin2016domain}. The $\mathcal{H}$-divergence between the source domain $D_S$ and target domain $D_{T}$ is computed as follows:
\begin{equation}
\begin{aligned}
    d_\mathcal{H}(D_S, D_{T}) = 2 \mathop{\mathrm{sup}}_{\mathcal{D} \sim \mathcal{H}} | \mathop{\mathrm{Pr}}_{x \sim D_S}[\mathcal{D}(x)=1] - \mathop{\mathrm{Pr}}_{x \sim D_{T}}[\mathcal{D}(x)=1] |
\end{aligned}
\label{apeq:dh}
\end{equation}
where $\mathcal{H}$ denotes the hypothesis space and $\mathcal{D}$ the discriminator. Initially, we quantify the target error relative to the source error, examining the discrepancy between labeling functions $f_S$ and $f_T$, and the divergence between distributions $\mathcal{P}_{D_{S}}$ and $\mathcal{P}_{D_{T}}$~\cite{ben2006analysis, ben2010theory}. This analysis establishes an initial bound on the discriminator's target error.
\begin{theorem}
\textit{For a hypothesis $\mathcal{H}$,}
\begin{equation}
\begin{aligned}
    \epsilon_{T}(\mathcal{H}) &\leq \epsilon_{S}(\mathcal{H}) + d(\mathcal{P}_{D_{S}}, \mathcal{P}_{D_{T}}) \\
    + &min\{E_{\mathcal{P}_{D_{S}}}[|f_{S}(x)-f_{T}(x)|], E_{\mathcal{P}_{D_{T}}}[|f_{S}(x)-f_{T}(x)|] \}
\end{aligned}
\end{equation}
\end{theorem}

\begin{proof}
\textit{Given that $\epsilon_{T}(\mathcal{H}) = \epsilon_{T}(\mathcal{H}, f_{T})$ and $\epsilon_{S}(\mathcal{H}) = \epsilon_{S}(\mathcal{H}, f_{S})$, let $\phi_{S}$ and $\phi_{T}$ denote the density functions of distributions $\mathcal{P}_{D_{S}}$ and $\mathcal{P}_{D_{T}}$, respectively, we can have:}
\begin{equation}
    \begin{aligned}
        \epsilon_{T}(\mathcal{H}) &= \epsilon_{T}(\mathcal{H}) + \epsilon_{S}(\mathcal{H}) - \epsilon_{S}(\mathcal{H}) + \epsilon_{S}(h,f_{T}) - \epsilon_{S}(\mathcal{H},f_{T}) \\ 
        &\leq \epsilon_{S}(\mathcal{H}) + |\epsilon_{S}(\mathcal{H},f_{T}) - \epsilon_{S}(\mathcal{H},f_{S})| \\
        &\quad + |\epsilon_{T}(\mathcal{H},f_{T}) - \epsilon_{S}(\mathcal{H},f_{T})| \\ 
        &\leq \epsilon_{S}(\mathcal{H}) + E_{\mathcal{P}_{D_{S}}}[|f_{S}(x)-f_{T}(x)|] \\
        &\quad + |\epsilon_{T}(\mathcal{H},f_{T}) - \epsilon_{S}(\mathcal{H},f_{T})| \\ 
        &\leq \epsilon_{S}(\mathcal{H}) + E_{\mathcal{P}_{D_{S}}}[|f_{S}(x)-f_{T}(x)|] \\
        &\quad + \int|\phi_{S}(x) - \phi_{T}(x)||\mathcal{H}(x)-f_{T}(x)|dx \\
        &\leq \epsilon_{S}(\mathcal{H}) + E_{\mathcal{P}_{D_{S}}}[|f_{S}(x)-f_{T}(x)|] + d(\mathcal{P}_{D_{S}}, \mathcal{P}_{D_{T}})
    \end{aligned}
\end{equation}
\end{proof}

In the first line, we can alternatively add and subtract $\epsilon_{T}(h,f_{S})$ instead of $\epsilon_{T}(h,f_{T})$, which theoretically yields a similar bound with respect to the target distribution $\mathcal{P}_{D_{T}}$ rather than $\mathcal{P}_{D_{S}}$. Choosing the lesser of the theoretical bounds provides a more advantageous constraint. However, estimating the error for the optimal hyperplane discriminator across arbitrary distributions remains an NP-hard problem. To counter this, we minimize a convex upper bound on the error, a common approach in classification, though it does not yield a valid upper bound on the target error. Therefore, we adopt the Jensen-Shannon (JS) divergence as an approximate measure of $\mathcal{H}$-divergence between $\mathcal{P}_{D_{S}}$ and $\mathcal{P}_{D_{T}}$.A minimal inter-domain divergence, detailed in~\cite{ganin2016domain}, signifies robust feature representation with diminished susceptibility to domain adaptation.
\begin{equation}
\begin{aligned}
    JS(\mathcal{P}_{D_{S}} || \mathcal{P}_{D_{T}}) &= \frac{1}{2}KL(\mathcal{P}_{D_{S}}|| \frac{\mathcal{P}_{D_{S}}+\mathcal{P}_{D_{T}}}{2}) \\
    &+ \frac{1}{2}KL(\mathcal{P}_{D_{T}}|| \frac{\mathcal{P}_{D_{S}}+\mathcal{P}_{D_{T}}}{2}),
\end{aligned}
\label{eq:js}
\end{equation}
where $KL(\cdot||\cdot)$ denotes the $Kullback-Leibler Divergence$. As shown in Fig. \ref{fig:divergence}, the BEVUDA++ shows significantly lower inter-domain divergence, indicating robust feature representation across domains. 

\section{Discussion on computational efficiency}
During the training phase, the additional computational overhead of the BEVUDA++ framework primarily arises from the integration of multi-geometric space features and uncertainty-guided depth estimation. The extra computational costs associated with multi-geometric space feature integration are comparable to those of STM3D~\cite{li2022unsupervised} and MTTrans~\cite{yu2022mttrans}. The main computational burden of uncertainty estimation lies in multiple forward passes through the depthnet. However, since this process does not involve complex backward propagation and the depthnet comprises only a few layers of linear and convolutional operations with a relatively small number of parameters, the increase in computational load amounts to only a few percentage points relative to the overall model.

During the inference phase, only the student model (GCS) participates and thus maintains computational efficiency. As uncertainty estimation is not performed during inference, there is no additional computational overhead or increase in model parameters compared to the original model (e.g., BEVDepth~\cite{li2022bevdepth}), which helps to preserve low latency crucial for real-time applications. Experiments demonstrate that both BEVUDA++ and BEVDepth achieve a frame rate of 24.3 FPS at INT8 precision on a V100 GPU.

Thus, despite a sacrifice in computational efficiency during training, the enhancements in inference efficiency and overall performance may well justify this trade-off. In practical applications, particularly in the field of autonomous driving, such performance improvements are likely critical, as they can significantly enhance detection accuracy across diverse environments and conditions.

\section{Experiments}
We carry out comprehensive experiments to validate the efficacy of the BEVUDA++. 
In Section~\ref{sec:4.1}, we detail the configuration of UDA scenarios and the specifics of our implementation. 
Section~\ref{sec:4.2} evaluates the cross-domain capabilities of BEVUDA++ across a variety of challenging environments.
A series of thorough ablation studies are presented in Section~\ref{ablation}, which dissects the contribution of each component within the framework. 
Lastly, in Section~\ref{sec:4.4}, we offer a qualitative analysis to better illustrate the benefits of our proposed BEVUDA++ framework.

\begin{figure}[t]
\includegraphics[width=0.48\textwidth]{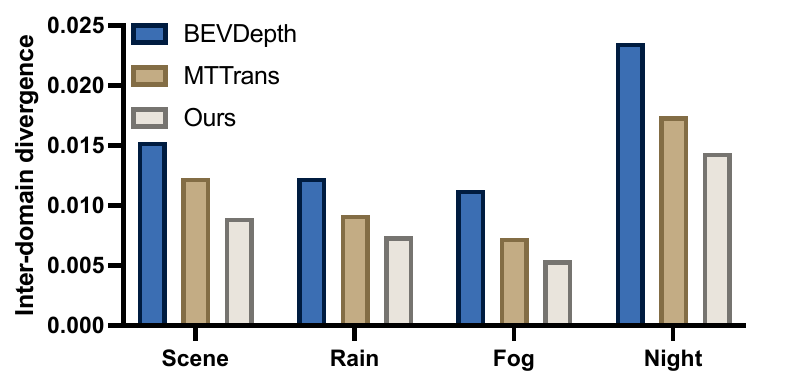}
\centering
\caption{Inter-domain divergence of the middle layer feature representations. The x-axis suggests the target domain.}
\label{fig:divergence}
\end{figure}

\subsection{Experimental setup}
\label{sec:4.1}
\subsubsection{Datasets and adaptation scenarios}
We have a comprehensive evaluation of our proposed framework on the large-scale nuScenes dataset~\cite{caesar2020nuScenes}, which is pivotal for advancing autonomous driving technologies. To facilitate progress in UDA for multi-view 3D object detection, we have partitioned the nuScenes dataset into distinct pairs of source-target domain data. We define and investigate three classical scenarios that embody the cross-domain challenges: \textbf{Scene}, \textbf{Weather}, and \textbf{Day-Night} transitions. Given the real-world scarcity of datasets featuring foggy conditions, we have augmented the available data by creating a synthetic foggy dataset, coined as Foggy-nuScenes, as shown in Fig. \ref{fig:fog}. This dataset was meticulously generated using a state-of-the-art fog simulation pipeline~\cite{sakaridis2018semantic} applied to the multi-view images within the nuScenes dataset~\cite{caesar2020nuScenes}. Our process synthesized foggy images across all scenes within the training, validation, and test sets, meticulously preserving the original 3D object detection annotations. To comprehensively cover the spectrum of visibility conditions, we constructed the Foggy-nuScenes dataset with five distinct levels of fog density: Foggy-1 through Foggy-5, with each successive level representing an increase in fog density. For the experiment's simplicity, we selected Foggy-1, Foggy-3, and Foggy-5 to serve as our experimental datasets. These datasets represent a range of weather conditions, ensuring that our experiments encompass significant domain shifts from the original nuScenes dataset~\cite{caesar2020nuScenes}, thereby providing a robust benchmark for weather adaptation and the evaluation of models under dynamically changing conditions.

\noindent\textbf{Scene Adaptation} We designate Boston as the source scene and execute UDA for the Singapore target domain. The ever-evolving scene layouts in autonomous driving lead to a domain gap across various scenes~\cite{yu2022cross, xu2020exploring}.

\noindent\textbf{Weathers Adaptation} Sunny conditions are set as the source domain data, with rainy and foggy weathers as the target domains. Real-world BEV detection must contend with these varied states and maintain reliability~\cite{cai2019exploring, wang2021exploring}.

\noindent\textbf{Day-Night Adaptation} We set daytime as the source domain and pursue UDA on the night-time target domain. Given the significant domain shift from day to night in camera-based methods, exploring domain adaptation for this scenario is crucial~\cite{sakaridis2021acdc, Wangetal2022}.

\noindent\textbf{Continues Changing Adaptation} We use sunny weather data as the source domain and incrementally denser fog conditions as the target domain, transitioning UDA from sunny to Foggy-1, Foggy-3, and Foggy-5 as fog intensity escalates.

\subsubsection{Implementation details}
The BEVUDA++ framework is developed on the BEVDepth architecture~\cite{li2022bevdepth}. Drawing from the literature~\cite{li2022bevdepth, reading2021categorical, huang2021bevdet, li2022unifying}, we utilize ResNet-50 and ResNet-101~\cite{he2016deep} as the backbone networks for feature extraction. Our configuration includes an image input size of $256\times 704$ and employs the same data augmentation techniques as~\cite{li2022bevdepth}. We opt for the AdamW optimizer~\cite{loshchilov2017decoupled} with a learning rate of 2e-4, foregoing any learning rate decay. For the training regime, we conduct source domain pre-training for 24 epochs and UDA experiments for 12 epochs, \textbf{excluding CBGS}~\cite{zhu2019class}. In the inference phase, we refrain from using test-time augmentation or model ensembles. All experiments are executed on NVIDIA Tesla V100 GPUs.

\subsubsection{Evaluation metrics}
Our evaluation metrics are in line with prior BEV research~\cite{li2022bevdepth,huang2021bevdet}, incorporating nuScenes Detection Score (NDS), mean Average Precision (mAP), and five True Positive (TP) metrics: mean Average Translation Error (mATE), mean Average Scale Error (mASE), mean Average Orientation Error (mAOE), mean Average Velocity Error (mAVE), and mean Average Attribute Error (mAAE).

\begin{figure}[t]
\includegraphics[width=0.48\textwidth]{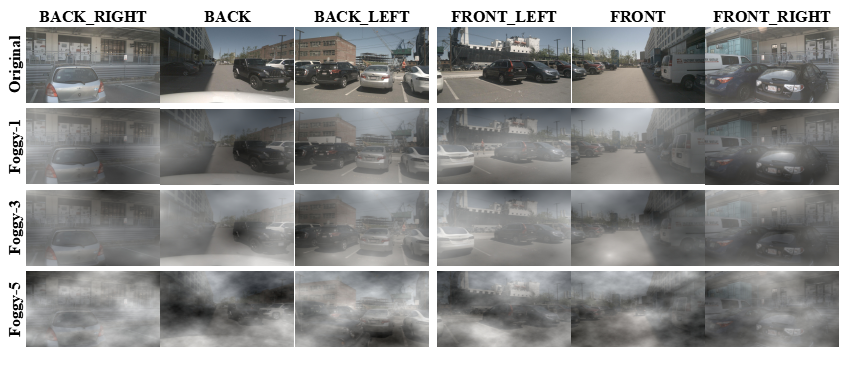}
\centering
\caption{The visualization of the Foggy-nuScenes dataset. The first row is the original multi-view images in nuScenes~\cite{caesar2020nuScenes}, and the last three rows demonstrate images of increased foggy degree.}
\label{fig:fog}
\end{figure}


\subsubsection{Generation of foggy nuScenes dataset}
We apply the fog simulation pipeline~\cite{sakaridis2018semantic} to the multi-view images provided in the nuScenes dataset~\cite{caesar2020nuScenes} to generate the Foggy-nuScenes dataset. Specifically, we generate synthetic foggy images for all scenes of training, validation, and test sets, which reserve the original annotation of the 3D object detection task. We utilize five different density degrees of the foggy simulator to construct the Foggy-nuScenes dataset, including Foggy-1, Foggy-2, Foggy-3, Foggy-4, and Foggy-5 (gradually increasing fog density). As shown in Fig. \ref{fig:fog}, we adopt Foggy-1, Foggy-3, and Foggy-5 as the experimental datasets for weather adaptation and continual changing scenario, which have an obvious domain gap with original nuScenes dataset~\cite{caesar2020nuScenes}.  

\begin{table*}[t]
  \centering
  \caption{Results of different methods for scene adaptation scenario on the validation set~\cite{caesar2020nuScenes}, from Boston to Singapore.}
    \setlength{\tabcolsep}{3mm}{
    \begin{tabular}{c|c|c|c|cccccc}
    \toprule
        \toprule
     Base/DA & Method & Backbone & \cellcolor{lightgray}NDS ↑& \cellcolor{lightgray}mAP ↑ & mATE ↓ & mASE ↓ & mAOE ↓ & mAVE ↓ & mAAE ↓ \\
    \midrule
    \midrule
     & BEVDet & ResNet-50 & 0.126 & 0.117 & 0.873 & 0.787 & 1.347& 1.302 & 0.666 \\
    Baseline & BEVDepth  & ResNet-50 & 0.174 & 0.115 & 0.888 & 0.412 & 1.031 & 1.056 & 0.527 \\
     & BEVDepth  & ResNet-101 & 0.187 & 0.115 & 0.874 & 0.391 & 0.944 & 1.021 & 0.501 \\
    \midrule
    \multirow{6}{*}{DA} & SFA  & ResNet-50 & 0.181 & 0.124 & 0.856 & 0.411 & 1.023 & 1.075 & 0.540 \\
     & STM3D  & ResNet-50 & 0.183 & 0.129 & 0.840 & 0.421 & 1.050 & 1.055 & 0.550 \\
      & MTTrans  & ResNet-50 & 0.197 & 0.140 & 0.796 & 0.425 & 1.132 & 1.088 & 0.542 \\
      & BEVUDA & ResNet-50 & 0.208 & 0.148 & 0.813 & 0.402 & 0.907 & 1.134 & 0.536 \\
     &BEVUDA & ResNet-101 & 0.211 & 0.166 & 0.758 & 0.427 & 1.127 & 1.108 & 0.535 \\
    \midrule
    \multirow{2}{*}{DA} 
    &BEVUDA++ & ResNet-50 & \textbf{0.209} & \textbf{0.152} & \textbf{0.663} & \textbf{0.423} & \textbf{1.060} & \textbf{1.275} & \textbf{0.561} \\
     &BEVUDA++ & ResNet-101 & \textbf{0.213} & \textbf{0.173} & \textbf{0.712} & \textbf{0.411} & \textbf{1.109} & \textbf{1.097} & \textbf{0.524} \\
    \bottomrule
        \bottomrule
    \end{tabular}
    }
  \label{tab:scene}
\end{table*}

\begin{table*}[t]
  \centering
  \caption{Results of different methods for scene adaptation scenario on the validation set~\cite{caesar2020nuScenes}, from Day to Night.} 
    \setlength{\tabcolsep}{3mm}{
    \begin{tabular}{c|c|c|c|cccccc}
    \toprule
        \toprule
     Base/DA & Method & Backbone &  \cellcolor{lightgray}NDS ↑&  \cellcolor{lightgray}mAP ↑ & mATE ↓ & mASE ↓ & mAOE ↓ & mAVE ↓ & mAAE ↓ \\
    \midrule
    \midrule
     & BEVDet & ResNet-50 & 0.010 & 0.009 & 0.990 & 0.977 & 1.078 & 1.509&0.984 \\
    Baseline & BEVDepth  & ResNet-50 & 0.050 & 0.012 & 0.042 & 0.646 & 1.129 & 1.705 & 0.915 \\
     & BEVDepth  & ResNet-101 & 0.062 & 0.036 & 1.033 & 0.706 & 0.973 & 1.447 & 0.895 \\
    \midrule
    \multirow{6}{*}{DA}  & SFA  & ResNet-50 & 0.092 & 0.032 & 0.995 & 0.556 & 0.993 & 1.480 & 0.948 \\
     & STM3D  & ResNet-50 & 0.070 & 0.035 & 0.979 & 0.549 & 1.063 & 1.587 & 0.937 \\
   & MTTrans  & ResNet-50 & 0.096 & 0.059 & 0.791 & 0.672 & 0.978 & 1.326 & 0.883 \\
    & BEVUDA & ResNet-50 & 0.132 & 0.054 & 0.711 & 0.465 & 1.072 & 1.504 & 0.772 \\
     & BEVUDA  & ResNet-101 & 0.188 & 0.127 & 0.189 & 0.484 & 0.820 & 1.784 & 0.711 \\
    \midrule
    \multirow{2}{*}{DA} 
     & BEVUDA++ & ResNet-50 & \textbf{0.135} & \textbf{0.080} & \textbf{0.797} & \textbf{0.472} & \textbf{0.880} & \textbf{1.385} & \textbf{0.525} \\
     & BEVUDA++  & ResNet-101 & \textbf{0.191} & \textbf{0.131} & \textbf{0.236} & \textbf{0.487} & \textbf{1.152} & \textbf{1.113} & \textbf{0.555} \\
    \bottomrule
        \bottomrule
    \end{tabular}%
    }
  \label{tab:day}%
\end{table*}%

\subsubsection{Backbone and baselines}
Our proposed BEVUDA++ method is benchmarked against existing BEV perception approaches BEVDet~\cite{huang2021bevdet} and BEVDepth~\cite{li2022bevdepth} to affirm its superior performance. Additionally, to showcase our unique design in tackling domain shifts in LSS-based multi-view 3D object detection, we reimplement notable 2D and mono-view 3D detection DA techniques on BEVDepth~\cite{li2022bevdepth}. As for the DA methods, we choose the state-of-the-art domain adaptation baselines include:
\begin{itemize}

    \item \textbf{SFA~\cite{wang2021exploring}} integrates DQFA for global context alignment and TDA for local feature matching, enhanced by a novel consistency loss for robust detection

    \item \textbf{STM3D~\cite{li2022unsupervised}} enhances Mono3D domain adaptation based on a teacher-student paradigm with quality-aware pseudo-label supervision.

    \item \textbf{MTTrans~\cite{yu2022mttrans}} leverages unlabeled target data for object detection via pseudo labels and a multi-level alignment process to alleviate the impact of domain shift.

    \item \textbf{BEVUDA~\cite{liu2024bevuda}} leverages a multi-latent space alignment teacher-student framework to narrow the domain gap under BEV-oriented UDA scenarios.
    
\end{itemize}

We further implement these methods on \textbf{BEVDepth} model for a fair comparison. It should also be noted that we conducted our proposed methods three times and reported the mean values to mitigate the bias introduced by randomness.

\subsection{Main results}
\label{sec:4.2}

\begin{table*}[t]
  \centering
  \caption{Results of different methods for weather adaptation scenarios on the validation set~\cite{caesar2020nuScenes}, from Sunny to Rainy and Foggy-3.} 
    \setlength{\tabcolsep}{3.5mm}{
    \begin{tabular}{c|c|c|ccc|ccc}
    \toprule
        \toprule
     \multirow{2}{*}{Base/DA} & \multirow{2}{*}{Method} & \multirow{2}{*}{Backbone} & & Target Rainy &  &  & Target Foggy-3 &  \\
     &  &  &  \cellcolor{lightgray}NDS ↑&  \cellcolor{lightgray}mAP ↑ & mATE ↓ &  \cellcolor{lightgray}NDS ↑&  \cellcolor{lightgray}mAP ↑ & mATE ↓  \\
    \midrule
    \midrule
     & BEVDet\ & ResNet-50 & 0.232 & 0.207 & 0.818 & 0.135 & 0.072 & 0.867  \\
    Baseline & BEVDepth & ResNet-50 & 0.268 & 0.196 & 0.824 & 0.159 & 0.079 & 0.882  \\
     & BEVDepth  & ResNet-101 & 0.272 & 0.212 & 0.842 & 0.202 & 0.122 & 0.804  \\
    \midrule
     \multirow{6}{*}{DA} & SFA  & ResNet-50 & 0.281 & 0.200 & 0.840 & 0.228 & 0.133 & 0.840  \\
     & STM3D  & ResNet-50 & 0.276 & 0.212 & 0.820 & 0.234 & 0.145 & 0.721  \\
     & MTTrans  & ResNet-50 & 0.283 & 0.199 & 0.800 & 0.244 & 0.159 & 0.733  \\
    & BEVUDA & ResNet-50 & 0.305 & 0.243 & 0.819 & 0.266 & 0.173 & 0.805  \\
     & BEVUDA  & ResNet-101 & 0.308 & 0.247 & 0.726 & 0.271 & 0.174 & 0.793  \\
    \midrule
    \multirow{2}{*}{DA} & BEVUDA++ & ResNet-50 & \textbf{0.307} & \textbf{0.245} & \textbf{0.778} & \textbf{0.269} & \textbf{0.174} & \textbf{0.790}  \\
     &BEVUDA++  & ResNet-101 & \textbf{0.312} & \textbf{0.251} & \textbf{0.702} & \textbf{0.273} & \textbf{0.178} & \textbf{0.784}  \\
    \bottomrule
        \bottomrule
    \end{tabular}%
    }
  \label{tab:weather}%
\end{table*}%

\begin{table}[t]
      \begin{center}
        \caption{Results of different methods for continuous changing adaptation scenario on nuScenes-Foggy, from Sunny to Foggy-1, Foggy-3, and Foggy-5 step by step. The metric is NDS}
        \vspace{-0.2cm}
        \setlength{\tabcolsep}{2mm}{
       	\begin{tabular}{c|c|ccc}
       	\toprule
               	\toprule
     Train on&Method(ResNet-50) & Foggy-1 & \cellcolor{lightgray}Foggy-3 & Foggy-5  \\
        \midrule
        \midrule
        \multirow{5}{*}{Sunny}& BEVDepth & 0.221 & 0.159 & 0.09\\
        & SFA & 0.256 & 0.220 & 0.149   \\
       & STM3D  & 0.251 & 0.238 & 0.183   \\
      & MTTrans  & 0.267 & 0.249 & 0.198   \\
      & BEVUDA & 0.283 & 0.271 & 0.219   \\
        & BEVUDA++ & \textbf{0.286} & \textbf{0.273} & \textbf{0.223}   \\
        \bottomrule
                \bottomrule
		\end{tabular}}
      \label{tab:fog}
      \end{center}
\end{table}


\noindent\textbf{Scene Adaptation} As shown in Table \ref{tab:scene} where DA means utilizing the domain adaptation method, BEVUDA++ outperforms all the baseline methods, which obviously exceeds BEVDepth~\cite{li2022bevdepth} of ResNet-50 and ResNet-101 backbone by 3.5\% and 2.6\% NDS. It thus demonstrates that our proposed method can effectively address the multi-geometric spaces domain shift caused by scene and environmental change. Compared with other SOTA DA methods, BEVUDA++ outperforms SFA and STM3D by 2.8\% and 2.6\% NDS, respectively. And it even improves mAP by 0.7\% compared with previous BEVUDA. The comparison further demonstrates that our proposed method is tailored for BEV object detection, and previous methods lack a special design for addressing the error accumulation problem in multi-geometric space.

\noindent\textbf{Day-Night Adaptation} The Day-Night adaptation is the most challenging scenario for camera-based methods. BEVUDA++ significantly improves the detection performance and solves the domain shift in the Night domain. In TABLE \ref{tab:day}, the tremendous domain gap makes baseline methods perform extremely poorly, with only 6.2\% NDS and 3.6\% mAP under ResNet-101. While BEVUDA++ can achieve 19.1\% NDS and 14.3\% mAP. Even compared with previous BEVUDA methods, it also achieves a superior improvement of more than 0.3\% NDS. Since previous DA methods ignore the inaccuracy and reliable depth estimation in Night data, it thus can not effectively extract target domain knowledge for transfer learning.

\noindent\textbf{Weathers Adaptation}  As shown in TABLE \ref{tab:weather}, in the Sunny to Foggy-3 adaptation scenario, BEVUDA++ outperforms other methods by a significant margin. Compared with STM3D and MTTrans, BEVUDA++ improves NDS by 3.5\% and 2.5\% since it can extract multi-space target domain knowledge to realize a better feature representation of target domain weather data. We also conduct experiments on Sunny to Rainy, BEVUDA++ also increases 0.2\% and 0.4\% NDS compared with BEVUDA. This further proves that our method can alleviate the model uncertainty under different target domain data distributions.

\begin{figure}[t]
\includegraphics[width=0.48\textwidth]{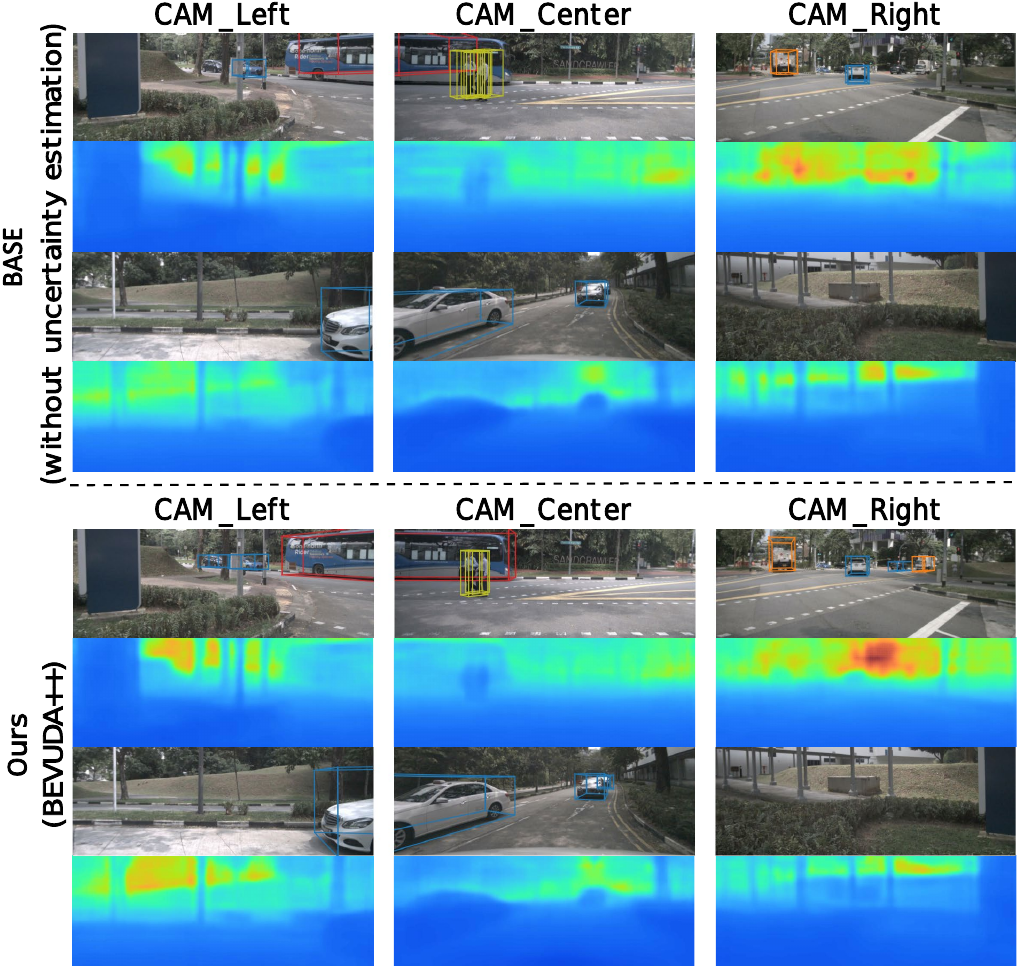}
\centering
\caption{Visualization of depth estimation results on nuScenes val set. The first row is the front view, while the second row is the back view.}
\label{fig:depth}
\vspace{-0.5cm}
\end{figure}

\noindent\textbf{Continues Changing Adaptation} As shown in TABLE \ref{tab:fog}, along with the increased foggy degree, the baseline method shows an obvious performance degradation. However, BEVUDA++ alleviates the gradually increased domain gap and outperforms MTTrans and BEVUDA by 2.5\% and 0.4\% NDS in the final Foggy-5 domain. The results prove that our method also can avoid error accumulation by reducing the model uncertainty to alleviate the continually changing domain shift.

\subsection{Ablation study}
\label{ablation}
To better reflect the role of each component in BEVUDA++, we conduct ablation experiments to analyze how each component can deal with domain shift for LSS-based BEV perception on \textbf{Sunny-Rainy} weather adaptation.

\noindent\textbf{The effectiveness of each component.}
In TABLE \ref{tab:abl}, vanilla BEVDepth ($Ex_{0}$) can only achieve 26.8\% NDS and 19.6\% mAP when the scenario is transformed from sunny to the rainy domain. For RDT, it transfers multi-latent space target domain knowledge to the student model, which is constructed by depth-aware information. As shown in $Ex_{1}$, the student model can absorb feature and pseudo-label level knowledge from RDT, thus improving NDS and mAP by 1.5\% and 3.5\%, respectively. 
As shown in $Ex_{2}$, the EMA updating only brings an improvement around 0.7\% NDS and 0.7\% mAP, which shows the effectiveness of our proposed UEMA in mitigating the significant domain gap.
By gradually incorporating more multi-space features ($Ex_{3} - Ex_{5}$) in the shared geometric space, our method will get a 2\% improvement in NDS, which demonstrates that it is crucial to jointly address all space domain shifts. 
When we combine RDT and GCS in BEVUDA++ ($Ex_{6}$), NDS reaches 30.5\% while mAP achieves 24.3\%. We can come to the conclusion that NDS and mAP continuously increase with the addition of each component in RDT and GCS, demonstrating that each of these modules is necessary and effective. It also proves that RDT and GCS can jointly address the domain shift accumulation.

\begin{figure*}[t]
\includegraphics[width=0.95\textwidth]{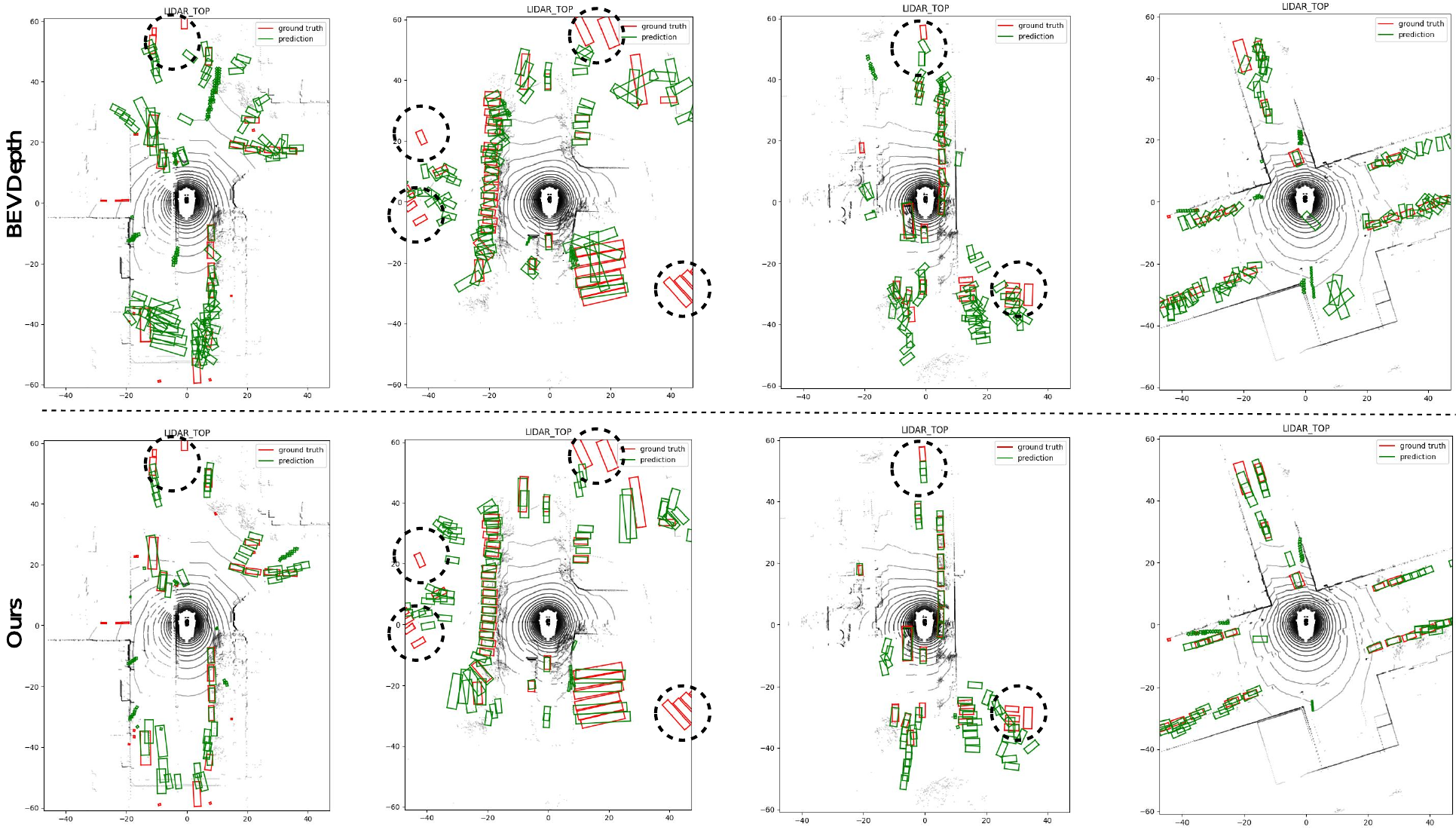}
\centering
\caption{Visualization results of 3D object detection on the nuScenes val set without NMS. Our proposed BEVUDA++ (bottom row) achieves more accurate and more certain predicted results than the backbone model (top row) while highlighting a few instances of failure cases.}
\label{fig:LiDAR}
        \vspace{-0.3cm}
\end{figure*}

\begin{table}[t]
      \begin{center}
        \caption{Ablation studies on the Sunny to Rainy scenario. DAT consists of three components: depth-aware information(DA), uncertainty-guided EMA(UEMA), and multi-latent space knowledge transfer(KT). For GAS, the geometric embedding space can be constructed by Bev(BA), image(IA), and voxel(VA) feature space.}
        \vspace{-0.3cm}

         \resizebox{\columnwidth}{!}{
        \setlength{\tabcolsep}{1mm}{
       	\begin{tabular}{c|ccc|ccc|cc}
       	\toprule
               	\toprule
     Name & DA & UEMA & KT& BA& IA & VA   &  \cellcolor{lightgray}NDS ↑&  \cellcolor{lightgray}mAP ↑ \\
        \midrule
        \midrule
        $Ex_{0}$  & - & - & - & - & - & -  & 0.268  & 0.196 \\
    
        \midrule
        $Ex_{1}$ & \Checkmark  & - & \Checkmark & - & - & - & 0.283  & 0.231 \\
       $Ex_{2}$ & \Checkmark & \Checkmark & \Checkmark  & - & - & - & 0.290  & 0.238 \\
       \midrule
       $Ex_{3}$ & - & - & -  & \Checkmark & - & -  &0.276  & 0.200\\
       $Ex_{4}$ & - & -  & - & \Checkmark & \Checkmark & -  & 0.282  & 0.204\\
       $Ex_{5}$ & - & -  & - & \Checkmark & \Checkmark & \Checkmark  & 0.288  & 0.207 \\
       \midrule
       $Ex_{6}$ & \Checkmark & \Checkmark  & \Checkmark & \Checkmark & \Checkmark & \Checkmark & \textbf{0.307}  & \textbf{0.245} \\
        		\bottomrule  
		\bottomrule  
		\end{tabular}}
  }
      \label{tab:abl}
      \end{center}
              \vspace{-0.3cm}
\end{table}

\noindent\textbf{Detailed ablation study of RDT.}
We study the effectiveness of depth-aware information composition and multi-geometric space knowledge transferring in RDT.
As shown in TABLE \ref{tab:abl_da}, only taking LiDAR ground truth to replace depth prediction ($Ex_{2-1}$) can improve 0.7\% NDS and 2.7\% mAP compared with $Ex_{0}$. The obviously increased mAP demonstrates that LiDAR data plays an important role in target domain-specific Voxel feature construction. However, due to the sparse property of LiDAR data, we utilize all dense depth predictions to composite sparse LiDAR. In ($Ex_{2-2}$),  NDS and mAP can achieve 27.9\% and 22.8\%, which only have limited improvement compared with $Ex_{2-1}$ since the original predictions contain noises. 
We further adopt traditional confidence scores to select reliable depth prediction. In ($Ex_{2-3}$), NDS and mAP can achieve 0.2\% and 0.1\% improvement compared with $Ex_{2-2}$, since confidence is trusting less in the pixel-wise cross-domain scenario and is not suitable to reflect the reliability of pixel-wise prediction. 
Therefore, we introduce uncertainty guidance to adaptively select more reliable and task-relevant depth predictions. $Ex_{2}$ has obvious performance progress compared with $Ex_{2-1}$ - $Ex_{2-3}$, demonstrating the uncertainty mechanism can reflect the reliability of depth prediction. It thus reduces the noises of depth information and can further ease the domain shift influence. 
As shown in TABLE \ref{tab:abl_kt}, transferring target domain knowledge in different geometric spaces can be beneficial to RDT. With pseudo label, BEV, voxel, and image feature transferred between RDT and student model, mAP is gradually improved from 19.6\% to 23.5\%. 
The improved performance proves that the transferred multi-space target domain knowledge is essential for the student model to decrease data distribution distance between two domains.

\begin{figure*}[t]
\includegraphics[width=0.95\textwidth]{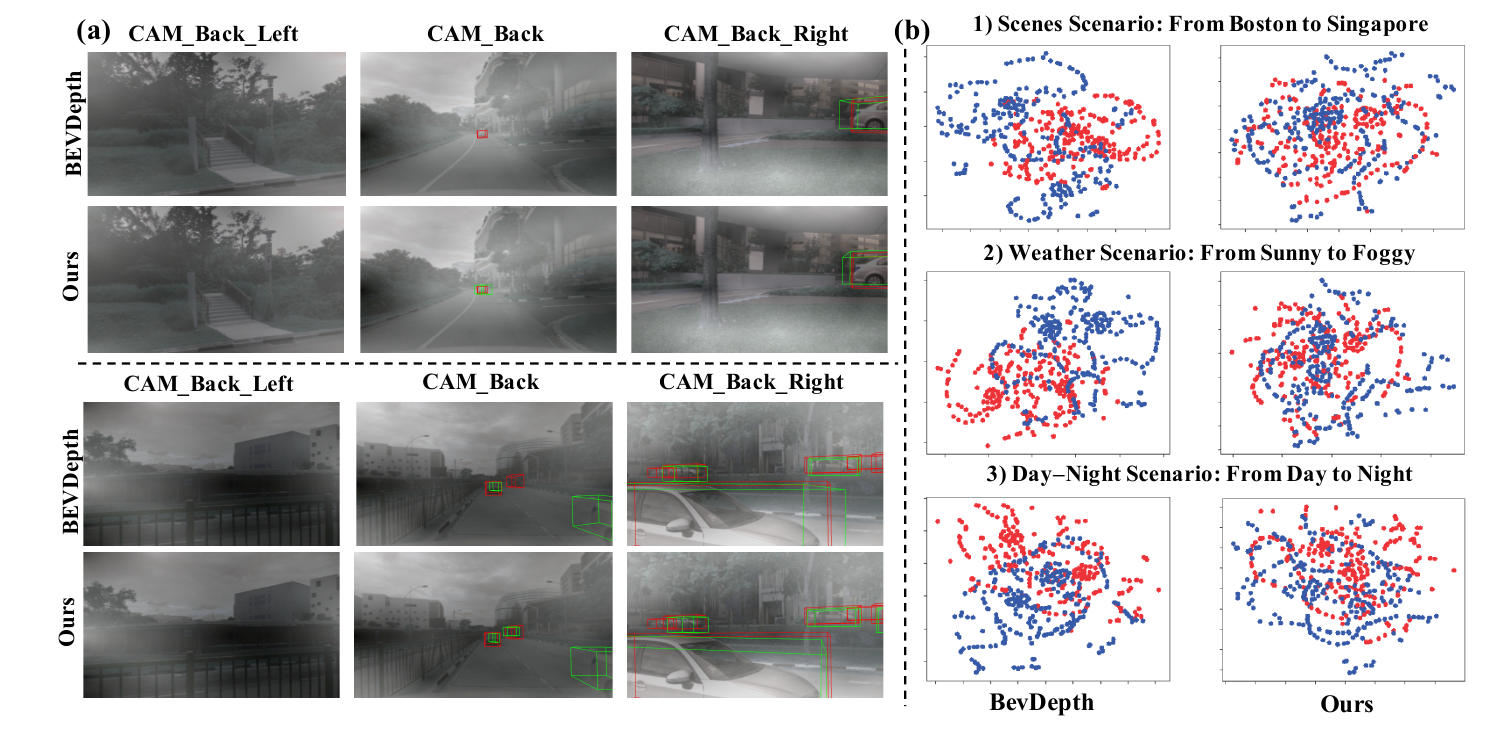}
\centering
\caption{Visualizations on the benefits of our proposed method. (a) Camera space results: The visualization of BevDepth~\cite{li2022bevdepth}(top row) and our proposed BEVUDA++(bottom row). The results are visualized on the weather adaptation scenario. (b) Visualization of feature distributions using T-SNE~\cite{van2013barnes}.}
\label{fig:vis}
        \vspace{-0.3cm}
\end{figure*}

\subsection{Qualitative analysis}
\label{sec:4.4}

We first visualize the depth estimation results on nuScenes val set in Fig. \ref{fig:depth}. Compared to base methods without uncertainty estimation, our proposed BEVUDA++ significantly enhances depth estimation performance, yielding sharper edges and more defined shapes. The integration of uncertainty estimation in BEVUDA++ aids in better distinguishing between areas of high and low confidence in-depth predictions, which enhances the reliability and visual clarity of the depth maps. These advancements in in-depth estimation contribute to substantially more precise 3D object detection.

Moreover, Fig. \ref{fig:LiDAR} illustrates the 3D detection results on the nuScenes val set under the Sunny to Rainy scenario. It is important to note that in order to more effectively showcase the efficacy of our proposed methods, we have opted not to perform Non-Maximum Suppression (NMS). Observations reveal that our BEVUDA++ not only secures higher IoU but also delivers more reliable predictions with fewer false positives. These improvements are attributed to the uncertainty estimation for reliable depth predictions and further refine the robustness of the EMA process.

We further present some visualization results of the camera space, as shown in Fig. \ref{fig:vis} (a). It is quite clear that the BEVDpeth fails to locate the objects well, while BEVUDA++ yields more accurate localization results as its prediction overlaps better with the ground truth. We can also observe that BEVUDA++ can detect objects that baseline ignores, demonstrating the superiority of BEVUDA++ in object detection and presenting great potential in deploying to real-world autonomous driving applications. The visualization in Fig. \ref{fig:vis} (b) further verifies the explicit cross-domain ability of BEVUDA++. As a clear separation can be seen in the clusters of the blue-source and red-target dots produced by BEVDepth, the features generated by BEVUDA++ get closer distribution, further demonstrating the ability of our proposed method in addressing domain shift accumulation. In conclusion, these results present the great potential of BEVUDA++ in deploying to real-world autonomous driving applications.

\begin{table}[t]
      \begin{center}
        \caption{The ablation study on the effectiveness of each component in depth-aware information (DAT). Pred means directly utilizing depth prediction, and Con and UG mean adaptive confidence- and uncertainty-guided depth selection, respectively.}
        \vspace{-0.3cm}
         \resizebox{\columnwidth}{!}{
        \setlength{\tabcolsep}{1.2mm}{
       	\begin{tabular}{c|cccc|cc}
       	\toprule
               	\toprule
     Depth-aware: &  LiDAR & Pred & Con & UG &  \cellcolor{lightgray}NDS ↑&  \cellcolor{lightgray}mAP ↑ \\
        \midrule
        \midrule
        $Ex_{2-1}$& \Checkmark  & - & - & - &  0.275  & 0.223 \\
        $Ex_{2-2}$ & \Checkmark  &\Checkmark & -  & -  &  0.278  & 0.228 \\
        $Ex_{2-3}$ & \Checkmark  &\Checkmark &\Checkmark & -  &  0.280  & 0.229 \\
        $Ex_{2}$ & \Checkmark &\Checkmark & - & \Checkmark  & \textbf{0.290}  & \textbf{0.238} \\
        \bottomrule
                \bottomrule
		\end{tabular}}
  }
      \label{tab:abl_da}
      \end{center}
              \vspace{-0.5cm}
\end{table}

\vspace{-0.5cm}
\section{Analysis of failure modes} Based on the previous quantitative and qualitative analyses, it is evident that our proposed method substantially improves the performance of backbone models in UDA scenarios. However, as demonstrated in Fig. \ref{fig:LiDAR}, it still encounters significant challenges. Notably, the detection of heavily occluded 3D objects, particularly those distant from the center of the field of view as indicated by the black circles, proves difficult. Furthermore, the domain shift from day to night continues to cause significant performance degradation, attributable to the pronounced domain gap. This analysis exposes two primary failure modes: occlusion-related errors and inadequate adaptation to varying illumination conditions. Recent methods such as~\cite{chang2023bev,huang2023tri} focus on predicting 3D occupancy instead of bounding boxes and leverage V2X techniques to offer potential solutions to mitigate RGB domain shift effects. Additionally, the integration of spike cameras~\cite{liu2024unsupervised} could further reduce the domain gap in low-light conditions, suggesting a viable direction for future research to enhance detection robustness across diverse operational environments.

\begin{table}[t]
      \begin{center}
        \caption{The ablation study on the effectiveness of each component in Multi-latent space Knowledge Transfer. PL means transferring instance-level pseudo labels. BEV, Voxel, and Image stand for transferring on corresponding latent spaces.}
        \vspace{-0.3cm}
         \resizebox{\columnwidth}{!}{
        \setlength{\tabcolsep}{1mm}{
      	\begin{tabular}{c|cccc|cc}
      	\toprule
             	\toprule
     Latent Space: & PL & BEV & Voxel & Image &  \cellcolor{lightgray}NDS ↑&  \cellcolor{lightgray}mAP ↑ \\
        \midrule
        \midrule
        $Ex_{2-4}$& \Checkmark& - & - & -  & 0.280  & 0.213 \\
        $Ex_{2-5}$& \Checkmark&  \Checkmark & - & -  & 0.284 & 0.223 \\
        $Ex_{2-6}$& \Checkmark & \Checkmark &\Checkmark& -  & 0.286 & 0.231 \\
         $Ex_{2}$ & \Checkmark  & \Checkmark & \Checkmark&\Checkmark  & \textbf{0.290}  & \textbf{0.238} \\
		\bottomrule  
  		\bottomrule  
		\end{tabular}}
  }
      \label{tab:abl_kt}
      \end{center}
              \vspace{-0.5cm}
\end{table}

\section{Conclusion}
We introduce the innovative geometric-aware teacher-student framework, BEVUDA++, to mitigate domain shift accumulation in LSS-based BEV perception effectively. The RDT employs depth cues to craft dependable pseudo labels and imparts multi-spatial knowledge from the target domain to the student model. The GCS leverages this combined domain knowledge to narrow the distribution gap between domains. Together, RDT and GCS tackle domain shift accumulation with Uncertainty-guided EMA, propelling BEVUDA++ to SOTA achievements across three UDA challenges and a dynamic domain adaptation scenario. Despite the increased computational overhead during training due to the teacher-student setup, the student model maintains the computational efficiency during the inference period. 

{\small
\bibliographystyle{ieee_fullname}
\bibliography{egbib}
}

\begin{IEEEbiography}[{\includegraphics[width=1in,height=1.25in,clip,keepaspectratio]{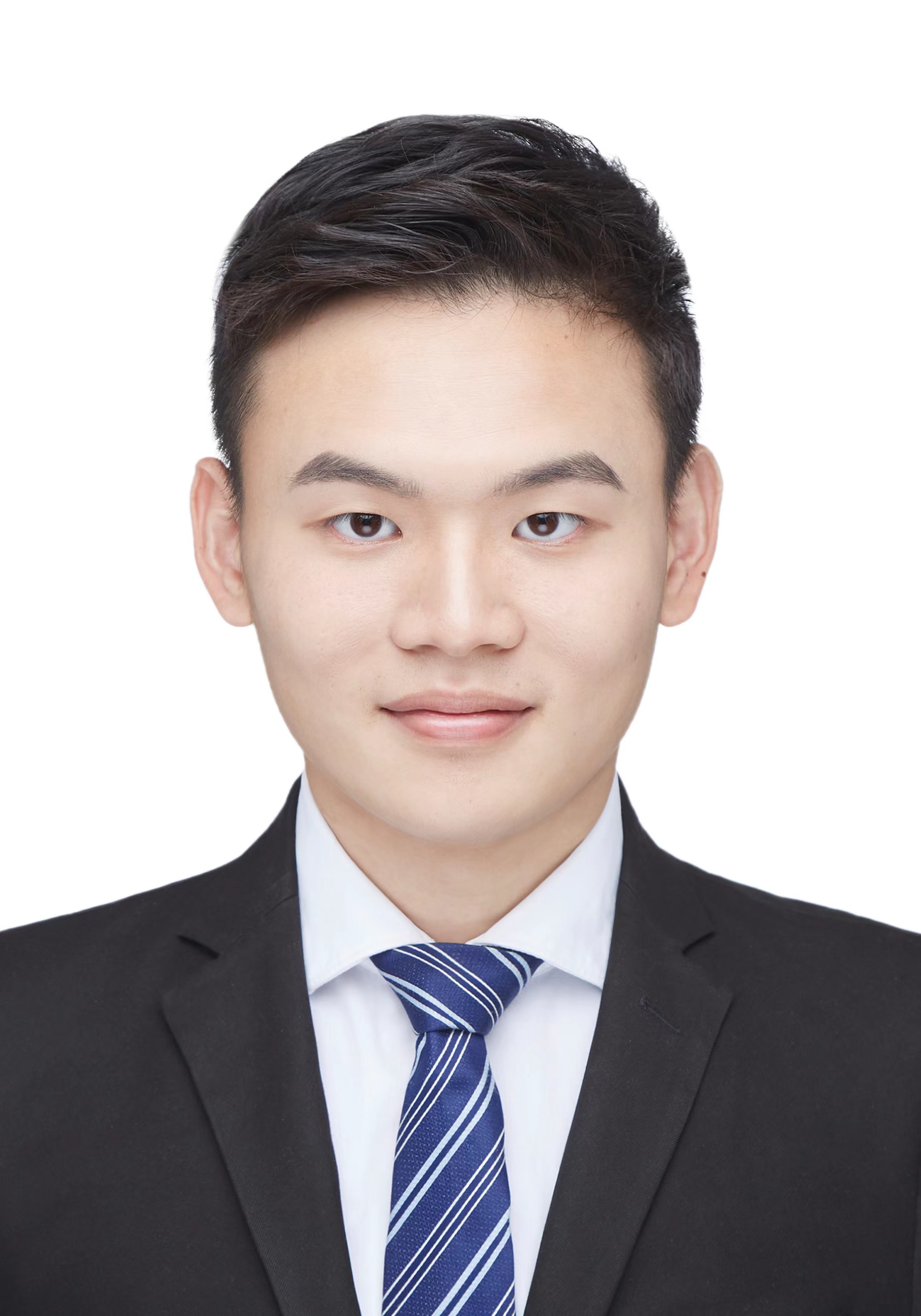}}]{Rongyu Zhang} is a dual Ph.D. student at Nanjing University and The Hong Kong Polytechnic University. He received M.Phil. degree at The Chinese University of Hong Kong, Shenzhen in 2023. He also received both B.E. degree and B.M. degree from the joint program of Beijing University of Posts and Telecommunications and Queen Mary University of London in 2021. His research interests include Federated Learning, Deep Learning, Efficient Learning for LLM and VLM.\end{IEEEbiography}

\vskip -2\baselineskip plus -1fil

\begin{IEEEbiography}[{\includegraphics[width=0.951in,height=1.15in,clip,keepaspectratio]{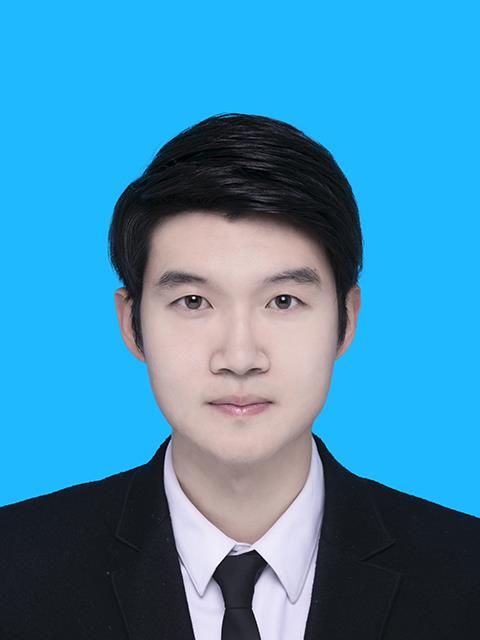}}]{Jiaming Liu}
 received the B.S. and M.S. degrees in Information and Communication Engineering from Beijing University of Posts and Telecommunications, China, in 2019 and 2022, receptively. Currently, he is pursuing his Ph.D. degree in Computer Science and Technology at Peking University. His research interests include out-of-distribution generalization, autonomous driving, and multimodal scene understanding. \end{IEEEbiography}

\vskip -2\baselineskip plus -1fil

\begin{IEEEbiography}[{\includegraphics[width=1.2in,height=1.3in,clip,keepaspectratio]{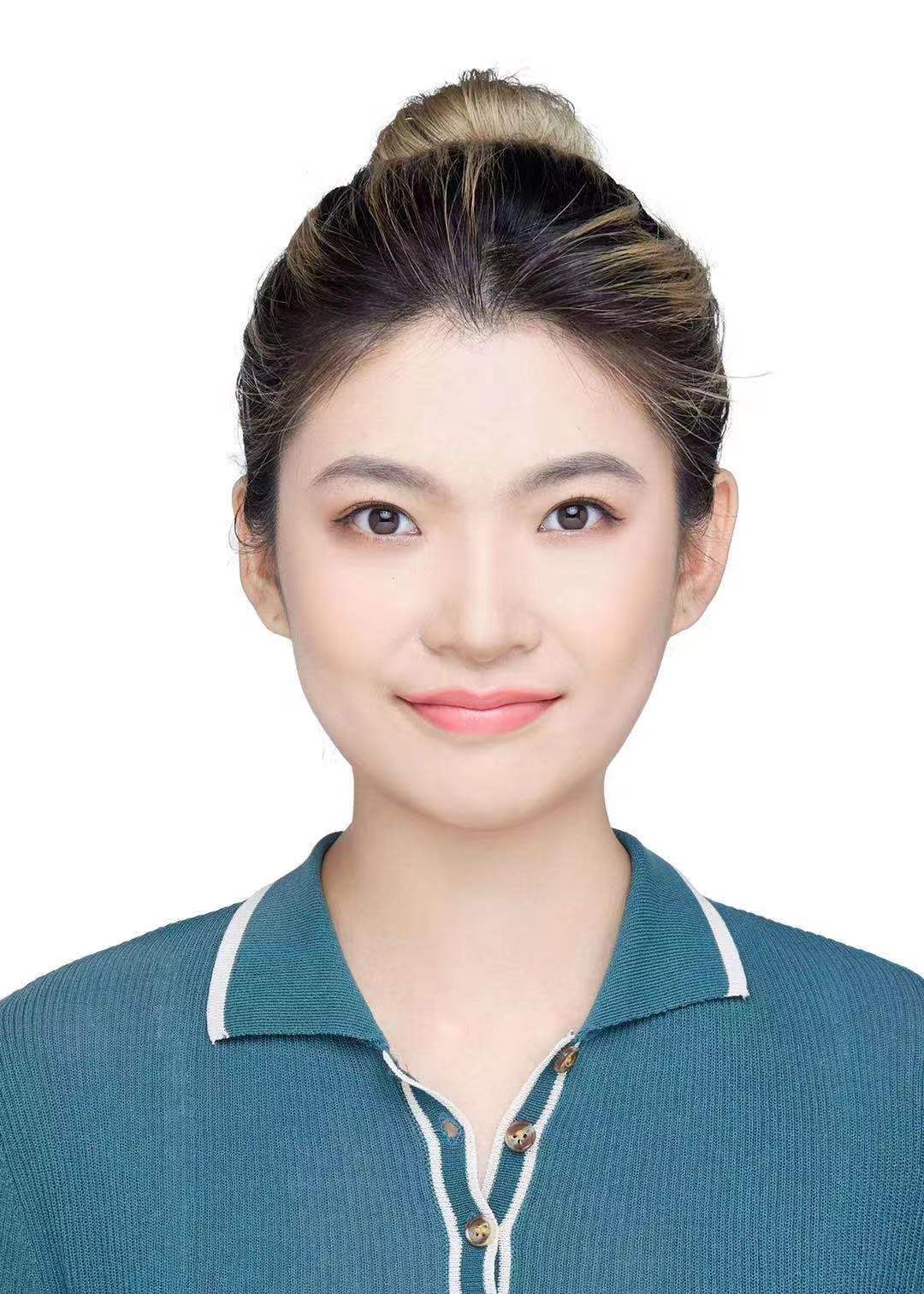}}]{Xiaoqi Li}
 received the B.S. degrees from Beijing University of Posts and Telecommunications, China, and the Master degree from Columbia University, USA. She is currently a first year phd student at Peking University. Her research interests focus on robot manipulation, embodied AI, and autonomous driving. \end{IEEEbiography}

\vskip -2\baselineskip plus -1fil
 
\begin{IEEEbiography}[{\includegraphics[width=1in,height=1.25in,clip,keepaspectratio]{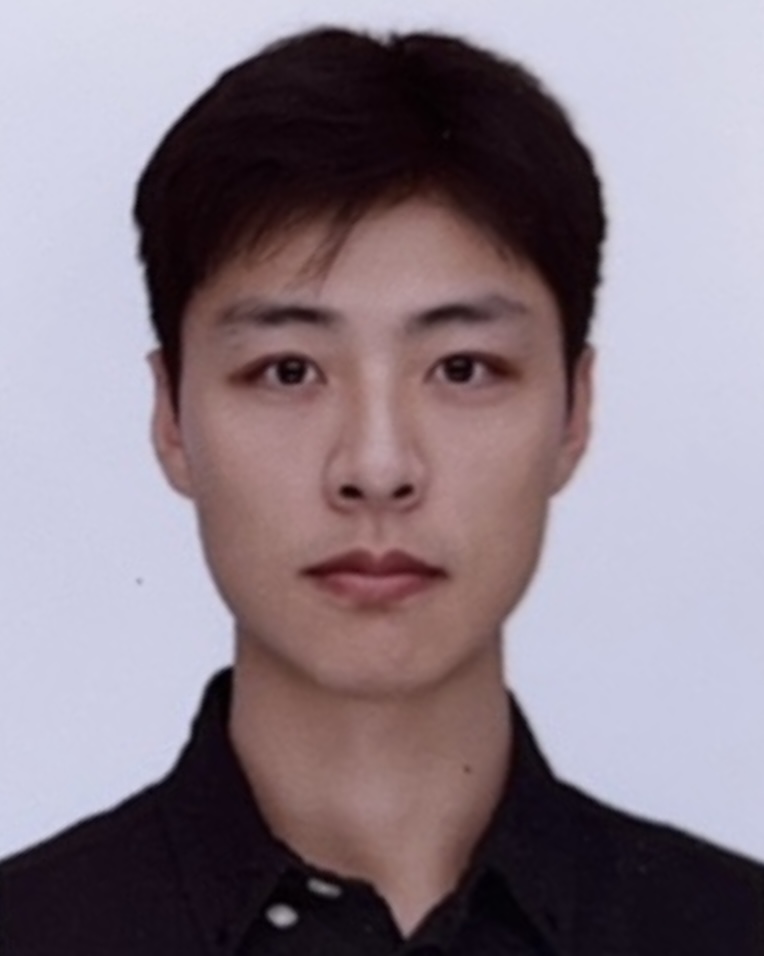}}]{Xiaowei Chi}
is a PhD student at Hong Kong University of Science and Technology, specializing in Robotics, Deep Learning, and Computer Vision. He received B.E. degree degree from the joint program of Beijing University of Posts and Telecommunications and Queen Mary University of London in 2021. With a passion for technology and its potential to transform society, Xiaowei's research focuses on developing algorithms for robotic perception and control using deep learning techniques.
\end{IEEEbiography}

\vskip -2\baselineskip plus -1fil

\begin{IEEEbiography}[{\includegraphics[width=1in,height=1.25in,clip,keepaspectratio]{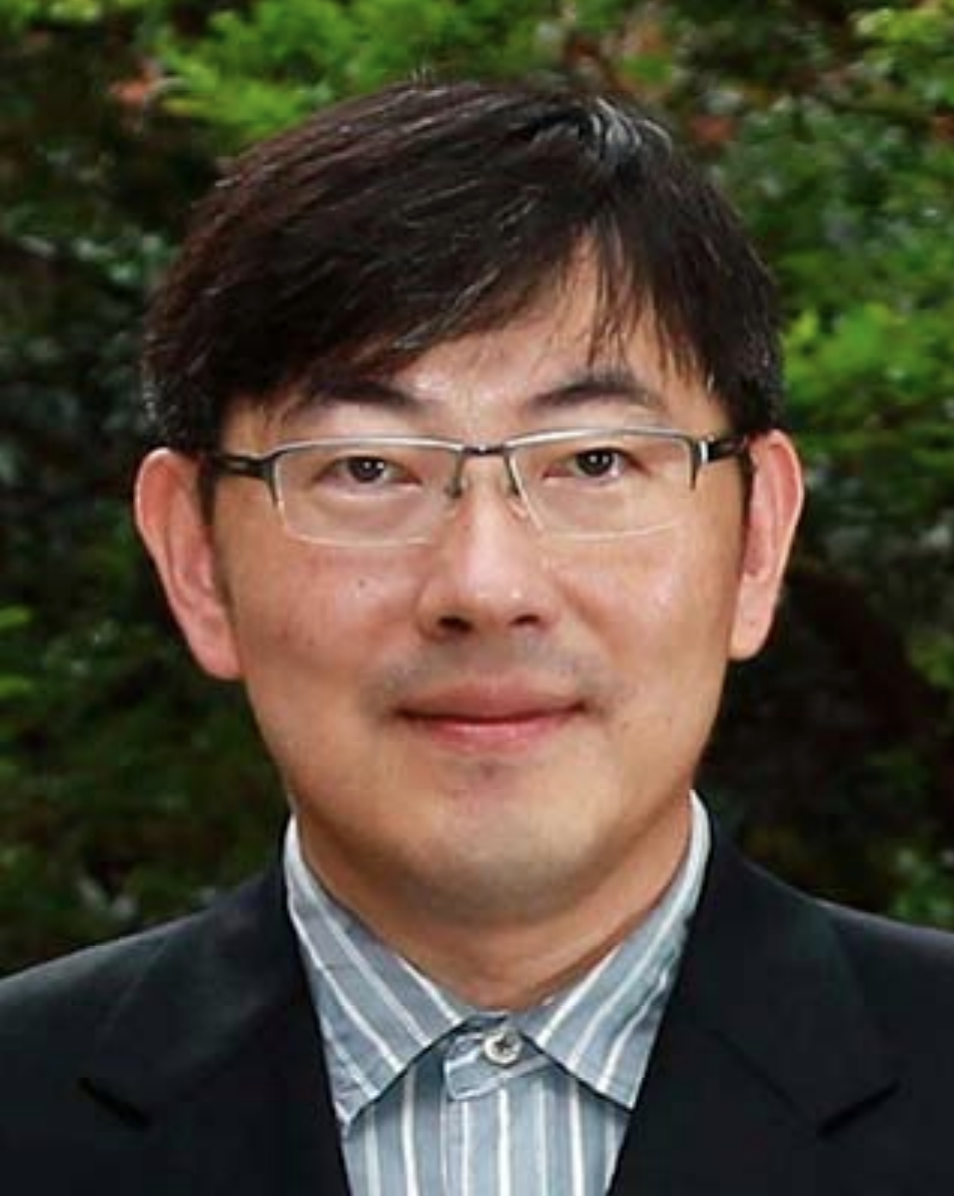}}]{Dan Wang} is a Professor of Department of Computing, The Hong Kong Polytechnic University. He received the Ph.D. degree from Simon Fraser University. His
research falls in general computer networking and systems, where he published in ACM SIGCOMM, ACM SIGMETRICS, and the IEEE INFOCOM, and many others. He is the Steering Committee Chair of IEEE/ACM IWQoS. His research interests include network architecture and QoS, smart building, and Industry 4.0.
\end{IEEEbiography}

\vskip -2\baselineskip plus -1fil

\begin{IEEEbiography}[{\includegraphics[width=0.951in,height=1.15in,clip,keepaspectratio]{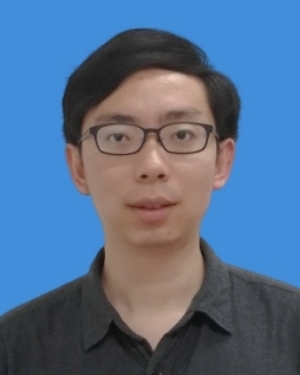}}]{Li Du}
(M’16) received his B.S degree from Southeast University, China and his Ph.D. degree in Electrical Engineering from University of California, Los Angeles.  Currently, he is an associate professor in the department of Electrical Science and Engineering at Nanjing University. His research includes analog sensing circuit design, in-memory computing design and high-performance AI processor for edge sensing.\end{IEEEbiography}

\vskip -2\baselineskip plus -1fil

\begin{IEEEbiography}[{\includegraphics[width=0.951in,height=1.15in,clip,keepaspectratio]{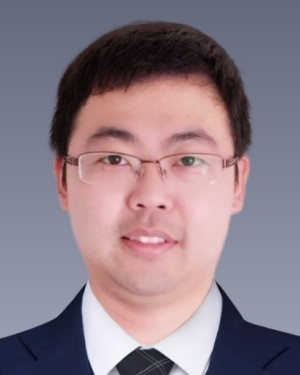}}]{Yuan Du}
(S’14-M’17-SM’21) received his B.S. degree from Southeast University (SEU), Nanjing, China, in 2009, his M.S. and his Ph.D. degree both from Electrical Engineering Department, University of California, Los Angeles (UCLA), in 2012 and 2016, respectively. Since 2019, he has been with Nanjing University, Nanjing, China, as an Associate Professor. His current research interests include designs of machine-learning hardware accelerators, high-speed inter-chip/intra-chip interconnects, and RFICs.
\end{IEEEbiography}

\vskip -2\baselineskip plus -1fil

\begin{IEEEbiography}[{\includegraphics[width=0.951in,height=1.15in,clip,keepaspectratio]{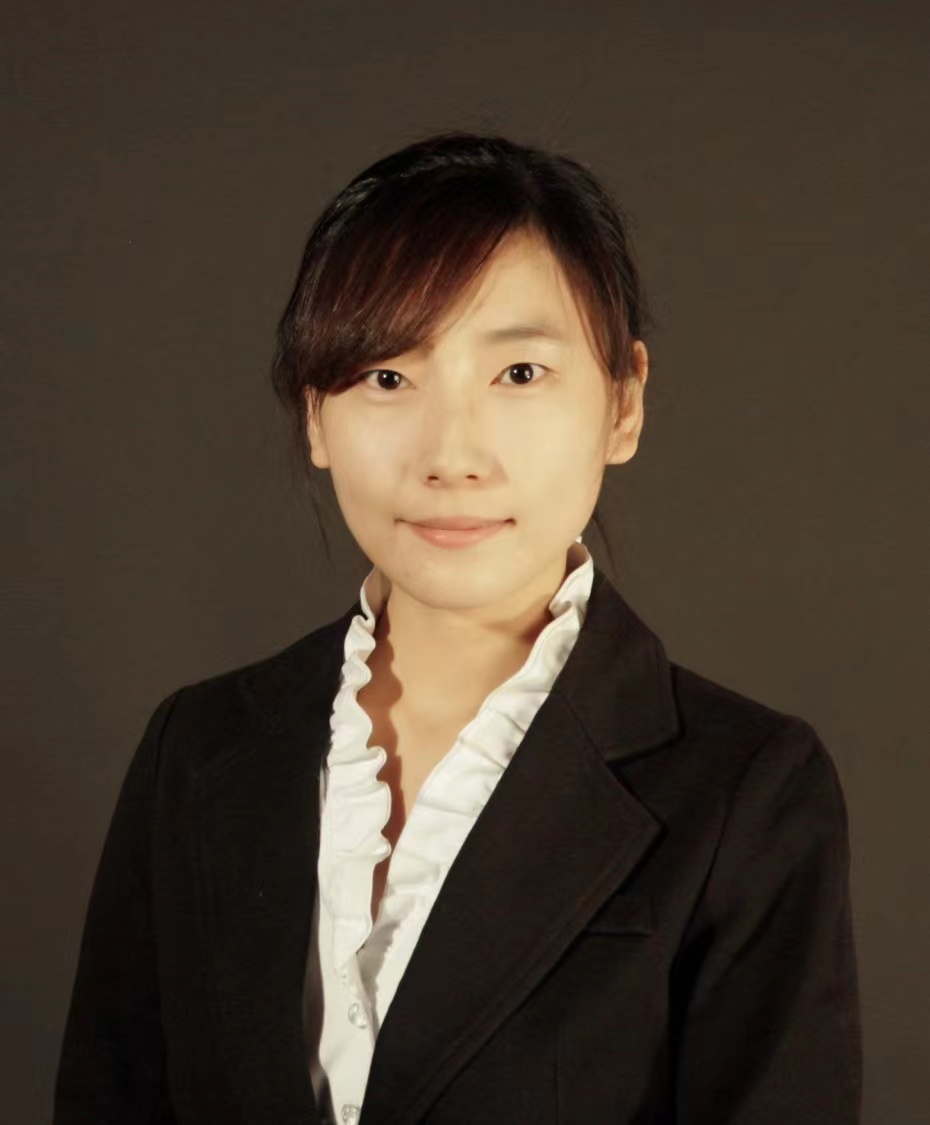}}]{Shanghang Zhang}
received her M.S. degree from Peking University, and her Ph.D. degree from Carnegie Mellon University in 2018. After that, she has been the postdoc research fellow at Berkeley AI Research Lab, UC Berkeley. Currently, she's an assistant professor at the School of Computer Science, Peking University. Her research focuses on machine learning generalization in the open world, including theory, algorithm, and system development, with applications to important IoT problems such as autonomous driving and robotics.\end{IEEEbiography}

\end{document}